\documentclass[10pt,twocolumn,letterpaper]{article}

\usepackage[accsupp]{axessibility}  
\usepackage{iccv}
\usepackage{times}
\usepackage{epsfig}
\usepackage{graphicx}
\usepackage{amsmath}
\usepackage{amssymb}
\usepackage{bbding}
\usepackage{booktabs}
\usepackage{amsthm}
\usepackage{subcaption}
\usepackage{bm}
\usepackage{bbm}
\usepackage{cases}
\usepackage{threeparttable}
\usepackage{diagbox}
\usepackage{multirow}
\usepackage{cuted}

\newtheorem{lemma}{Lemma}

\newcommand{\thickhline}{%
	\noalign {\ifnum 0=`}\fi \hrule height 1pt
	\futurelet \reserved@a \@xhline
}
\newcommand{\dprime}{{\prime\prime}}

\usepackage{xcolor}
\definecolor{newgreen}{RGB}{0,176,80}
\definecolor{newblue}{RGB}{0,176,240}

\usepackage[pagebackref=true,breaklinks=true,letterpaper=true,colorlinks,bookmarks=false]{hyperref}

\usepackage[capitalize]{cleveref}
\Crefname{section}{Section}{Sections}
\Crefname{table}{Table}{Tables}
\crefname{table}{Tab.}{Tabs.}

\usepackage{algorithm}
\usepackage{algorithmicx}
\usepackage{algpseudocode}
\floatname{algorithm}{Algorithm}

\iccvfinalcopy 

\def\iccvPaperID{****}

\ificcvfinal\fi

\begin{document}

\title{Towards Memory- and Time-Efficient Backpropagation for Training Spiking Neural Networks}
\author{
	Qingyan Meng$^{1,2}$, Mingqing Xiao$^3$, Shen Yan$^{4}$, Yisen Wang$^{3,5}$, Zhouchen Lin$^{3,5,6,}$\thanks{Corresponding author.} , Zhi-Quan Luo$^{1,2}$\\
	$^1$The Chinese University of Hong Kong, Shenzhen \
	$^2$Shenzhen Research Institute of Big Data\\
	$^3$National Key Lab. of General AI, School of Intelligence Science and Technology, Peking University \\
	$^4$Center for Data Science, Peking University \
	$^5$Institute for Artificial Intelligence, Peking University\\
	$^6$Peng Cheng Laboratory\\
	{\tt \small qingyanmeng@link.cuhk.edu.cn, \{mingqing\_xiao, yanshen, yisen.wang, zlin\}@pku.edu.cn,} \\ {\tt \small luozq@cuhk.edu.cn}
}
\maketitle
\ificcvfinal\thispagestyle{empty}\fi

\begin{abstract}
	Spiking Neural Networks (SNNs) are promising energy-efficient models for neuromorphic computing. For training the non-differentiable SNN methods, the backpropagation through time (BPTT) with surrogate gradients (SG) method has achieved high performance. However, this method suffers from considerable memory cost and training time during training. In this paper, we propose the Spatial Learning Through Time (SLTT) method that can achieve high performance while greatly improving training efficiency compared with BPTT. First, we show that the backpropagation of SNNs through the temporal domain contributes just a little to the final calculated gradients. Thus, we propose to ignore the unimportant routes in the computational graph during backpropagation. 
	The proposed method reduces the number of scalar multiplications and achieves a small memory occupation that is independent of the total time steps. Furthermore, we propose a variant of SLTT, called SLTT-K, that allows backpropagation only at $K$ time steps, then the required number of scalar multiplications is further reduced and is independent of the total time steps. Experiments on both static and neuromorphic datasets demonstrate superior training efficiency and performance of our SLTT. In particular, our method achieves state-of-the-art accuracy on ImageNet, while the memory cost and training time are reduced by more than 70\% and 50\%, respectively, compared with BPTT. Our code is available at \url{https://github.com/qymeng94/SLTT}.

\end{abstract}

\vspace{-5pt}
\section{Introduction}
\label{sec:intro}

\begin{figure}[h]
	\centering
	\includegraphics[width=0.95\columnwidth]{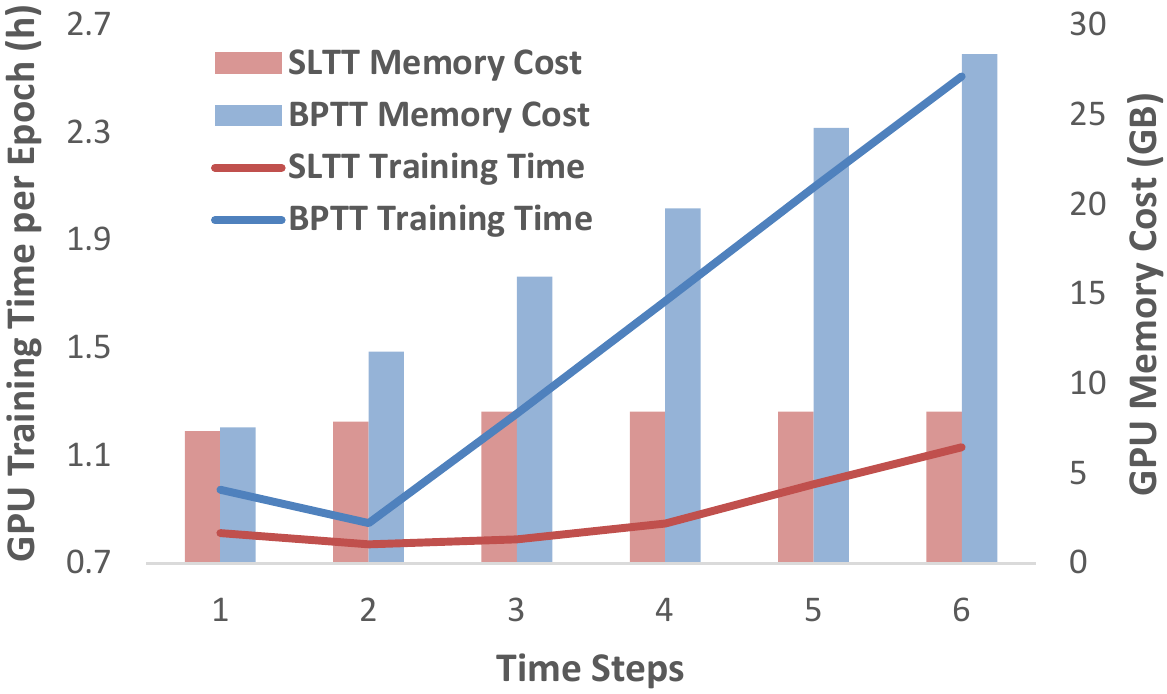} 
	\caption{ \small The training time and memory cost comparison between the proposed SLTT-1 method and the BPTT with SG method on ImageNet. SLTT-1 achieves similar accuracy as BPTT, while owning better training efficiency than BPTT both theoretically and experimentally. Please refer to \cref{sec:SLTT,sec:experiments} for details.
	}
	\label{fig:time_mem}	
\end{figure}

Regarded as the third generation of neural network models \cite{maass1997networks}, Spiking Neural Networks (SNNs) have recently attracted wide attention. SNNs imitate the neurodynamics of power-efficient biological networks, where neurons communicate through spike trains (\ie, time series of spikes). A spiking neuron integrates input spike trains into its membrane potential. After the membrane potential exceeds a threshold, the neuron fires a spike and resets its potential \cite{gerstner2014neuronal}.
The spiking neuron is active only when it experiences spikes, thus enabling event-based computation. This characteristic makes SNNs energy-efficient when implemented on neuromorphic chips~\cite{merolla2014million,davies2018loihi,pei2019towards}. 
As a comparison, the power consumption of deep Artificial Neural Networks (ANNs) is substantial.

The computation of SNNs with discrete simulation can share a similar functional form as recurrent neural networks (RNNs) \cite{neftci2019surrogate}. The unique component of SNNs is the non-differentiable threshold-triggered spike generation function. The non-differentiability, as a result, hinders the effective adoption of gradient-based optimization methods that can train RNNs successfully. Therefore, SNN training is still a challenging task. Among the existing SNN training methods, backpropagation through time (BPTT) with surrogate gradient (SG) \cite{cramer2022surrogate,sengupta2019going} has recently achieved high performance on complicated datasets in a small number of time steps (\ie, short length of spike trains). The BPTT with SG method defines well-behaved surrogate gradients to approximate the derivative of the spike generation function. Thus the SNNs can be trained through the gradient-based BPTT framework \cite{werbos1990backpropagation}, just like RNNs. With such framework, gradients are backpropagated through both the layer-by-layer spatial domain and the temporal domain. Accordingly, BPTT with SG suffers from considerable memory cost and training time that are proportional to the network size and the number of time steps. The training cost is further remarkable for large-scale datasets, such as ImageNet.

In this paper, we develop the Spatial Learning Through Time (SLTT) method that can achieve high performance while significantly reducing the training time and memory cost compared with the BPTT with SG method. We first decompose the gradients calculated by BPTT into spatial and temporal components. With the decomposition, the temporal dependency in error backpropagation is explicitly presented. 
We then analyze the contribution of temporal information to the final calculated gradients, and propose the SLTT method to delete the unimportant routes in the computational graph for backpropagation. In this way, the number of scalar multiplications is reduced; thus, the training time is reduced. SLTT further enables online training by calculating gradient instantaneously at each time step, without the requirement of storing information of other time steps. Then the memory occupation is independent of the number of total time steps, avoiding the significant training memory costs of BPTT.
Due to the instantaneous gradient calculation, we also propose the SLTT-K method that conducts backpropagation only at $K$ time steps. SLTT-K can further reduce the time complexity without performance loss.
With the proposed techniques, we can obtain high-performance SNNs with superior training efficiency. The wall-clock training time and memory costs of SLTT-1 and BPTT on ImageNet under the same experimental settings are shown in \cref{fig:time_mem}. 
Formally, our contributions include:
\begin{itemize}
	\item[1.] Based on our analysis of error backpropagation in SNNs, we propose the Spatial Learning Through Time (SLTT) method to achieve better time and memory efficiency than the commonly used BPTT with SG method. Compared with the BPTT with SG method, the  number of scalar multiplications is reduced, and the training memory is constant with the number of time steps, rather than grows linearly with it.
	
	\item[2.] Benefiting from our online training framework, we propose the SLTT-K method that further reduces the time complexity of SLTT. The required number of scalar multiplication operations is reduced from $\Omega(T)$\footnote{$f(x)=\Omega(g(x))$ means that there exist $c>0$ and $n>0$, such that $0\le cg(x)\le f(x)$ for all $x\ge n$.} to $\Omega(K)$, where $T$ is the number of total time steps, and $K<T$ is the parameter indicating the number of time steps to conduct backpropagation.
	
	\item[3.] Our models achieve competitive SNN performance with superior training efficiency on CIFAR-10, CIFAR-100, ImageNet, DVS-Gesture, and DVS-CIFAR10 under different network settings or large-scale network structures. On ImageNet, our method achieves  state-of-the-art accuracy while the memory cost and training time are reduced by more than 70\% and 50\%, respectively, compared with BPTT.
\end{itemize}
\section{Related Work}
\label{sec:related}
\paragraph{The BPTT Framework for Training SNNs.}
A natural methodology for training SNNs is to adopt the gradient-descent-based BPTT framework, while assigning surrogate gradients (SG) to the non-differentiable spike generation functions to enable meaningful gradient calculation \cite{zenke2021remarkable,neftci2019surrogate,wu2018spatio,wu2019direct,shrestha2018slayer,huh2017gradient,ma2023exploiting}. Under the BPTT with SG framework, many effective techniques have been proposed to improve the performance, such as threshold-dependent batch normalization \cite{zheng2020going}, carefully designed surrogate functions \cite{li2021differentiable} or loss functions \cite{guo2022recdis,deng2022temporal}, SNN-specific network structures \cite{fang2021sew}, and trainable parameters of neuron models \cite{fang2021incorporating}. Many works conduct multi-stage training, typically including an ANN pre-training process, to reduce the latency (\ie, the number of time steps) for the energy efficiency issue, while maintaining competitive performance \cite{rathi2019enabling,rathi2020diet,chowdhurytowards,chowdhury2021one}. The BPTT with SG method has achieved high performance with low latency on both static \cite{fang2021sew,guo2022reducing} and neuromorphic \cite{li2022neuromorphic,deng2022temporal} datasets.
However, those approaches need to backpropagate error signals through both temporal and spatial domains, thus suffering from high computational costs during training \cite{deng2020rethinking}. In this work, we reduce the memory and time complexity of the BPTT with SG framework with gradient approximation and instantaneous gradient calculation, while maintaining the same level of performance.

\paragraph{Other SNN Training Methods.} 
The ANN-to-SNN conversion method \cite{yan2021near,deng2021optimal,sengupta2019going,rueckauer2017conversion,han2020rmp,han2020deep,ding2021optimal} has recently yielded top performance, especially on ImageNet \cite{li2021free,meng2022ann,bu2022optimal}. This method builds a connection between the firing rates of SNNs and some corresponding ANN outputs. With this connection, the parameters of an SNN are directly determined from the associated ANN. Despite the good performance, the required latency is much higher compared with the BPTT with SG method. This fact hurts the energy efficiency of SNN inference \cite{davidson2021comparison}. Furthermore, the conversion method is not suitable for neuromorphic data. Some gradient-based direct training methods find the equivalence between spike representations (\eg, firing rates or first spike times) of SNNs and some differentiable mappings or fixed-point equations \cite{mostafa2017supervised,zhou2019temporal,xiao2021ide,meng2022training,thiele2019spikegrad,wu2021training,wu2021tandem,xiao2022spide,yang2022training}.
Then the spike-representation-based methods train SNNs by gradients calculated from the corresponding mappings or fixed-point equations.
Such methods have recently achieved competitive performance, but still suffer relatively high latency, like the conversion-based methods. To achieve low latency, our work is mainly based on the BPTT with SG method and then focuses on the training cost issue of BPTT with SG.

\begin{figure}[t]
	\centering
	\includegraphics[width=0.78\columnwidth]{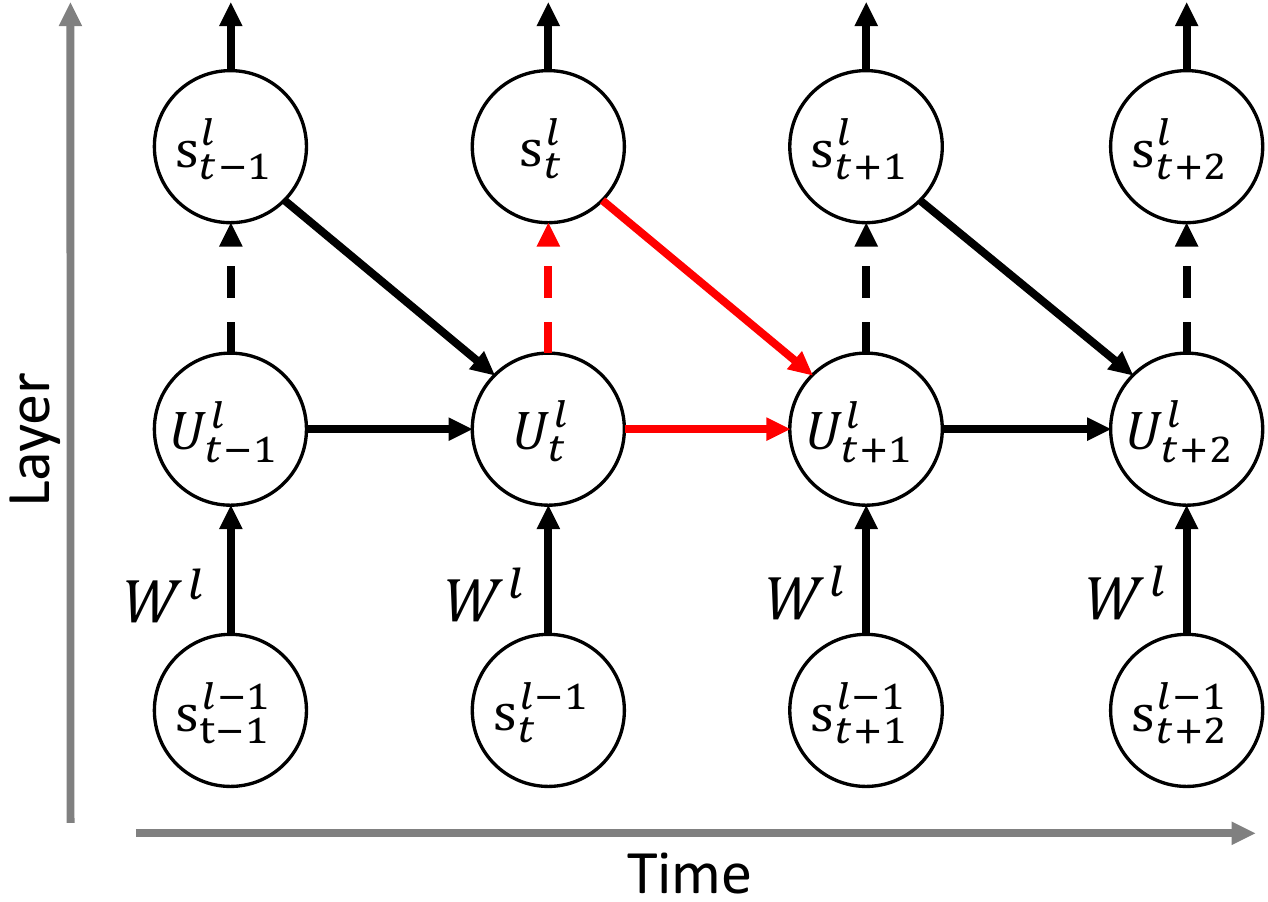} 
	\caption{\small Computational graph of multi-layer SNNs. Dashed arrows represent the non-differentiable spike generation functions.}
	\label{fig:bptt}	
\end{figure}

\paragraph{Efficient Training for SNNs.} 
Several RNN training methods pursue online learning and constant memory occupation agnostic time horizon, such as real time recurrent learning \cite{williams1989learning} and forward propagation through time \cite{kag2021training}. Inspired by them, some SNN training methods \cite{zenke2018superspike,zenke2021brain,bellec2020solution,bohnstingl2022online,yin2021accurate} apply similar ideas to achieve memory-efficient and online learning. However, such SNN methods cannot scale to large-scale tasks due to some limitations, such as using feedback alignment \cite{nokland2016direct}, simple network structures, and still large memory costs although constant with time.
\cite{kaiser2020synaptic} ignores temporal dependencies of information propagation to enable local training with no memory overhead for computing gradients.
They use similar ways as ours to approximate the gradient calculation, but do not verify the reasonableness of the approximation, and cannot achieve comparable accuracy as ours, even for simple tasks.
\cite{perez2021sparse} presents the sparse SNN backpropagation algorithm in which gradients only backpropagate through ``active neurons'', that account for a small number of the total, at each time step. However, \cite{perez2021sparse} does not consider large-scale tasks, and the memory grows linearly with the number of time steps. Recently, some methods \cite{yang2022training,xiao2022online} have achieved satisfactory performance on large-scale datasets with time steps-independent memory occupation. Still, they either rely on pre-trained ANNs and cannot conduct direct training\cite{yang2022training}, or do not consider reducing time complexity and require more memory than our work due to tracking presynaptic activities \cite{xiao2022online}.
Our work can achieve state-of-the-art (SOTA) performance while maintaining superior time and memory efficiency compared with other methods.

\section{Preliminaries}

\subsection{The Leaky Integrate and Fire Model}
A spiking neuron replicates the behavior of a biological neuron which integrates input spikes into its membrane potential $u(t)$ and transmits spikes when the potential $u$ reaches a threshold. Such spike transmission is controlled via some spiking neural models. In this paper, we consider a widely adopted neuron model, the leaky integrate and fire (LIF) model \cite{burkitt2006review}, to characterize the dynamics of $u(t)$:
\begin{equation} 
\tau \frac{\mathrm{d} u(t)}{\mathrm{d} t}=-(u(t)-u_{rest}) + R\cdot I(t), \ \text{when} \ u(t)<V_{th},
\end{equation}
where $\tau$ is the time constant, $R$ is the resistance, $u_{rest}$ is the resting potential, $V_{th}$ is the spike threshold, and $I$ is the input current which depends on received spikes. The current model is given as 
$
I(t)=\sum_{i}w_i^\prime s_i(t) + b^\prime,
$
where $w_i^\prime$ is the weight from neuron-$i$ to the target neuron, $b^\prime$ is a bias term, and $s_i(t)$ is the received train from neuron-$i$. $s_i(t)$ is formed as $s_i(t)=\sum_{f}\delta(t-t_{i,f})$, in which $\delta(\cdot)$ is the Dirac delta function and $t_{i,f}$ is the $f$-th fire time of neuron-$i$. Once $u\ge V_{th}$ at time $t_f$, the neuron output a spike, and the potential is reset to $u_{rest}$. The output spike train is described as $s_{out}(t)=\sum_{f}\delta(t-t_{f})$.

In application, the discrete computational form of the LIF model is adopted. With $u_{rest}=0$, the discrete LIF model can be described as
\begin{equation} \label{eqn:lif}
\left\{ 
\begin{aligned}
&u[t]=(1-\frac{1}{\tau})v[t-1] + \sum_{i}w_i s_i[t] + b, \\[-3pt]
&s_{out}[t]=H(u[t]-V_{th}),  \\
&v[t]=u[t]-V_{th} s_{out}[t],
\end{aligned}
\right. 
\end{equation}
where $t\in\{1,2,\cdots,T\}$ is the time step index, $H(\cdot)$ is the Heaviside step function, $s_{out}[t], s_{i}[t] \in \{0,1\}$, $v[t]$ is the intermediate value representing the membrane potential before being reset and $v[0]=0$, and $w_i$ and $b$ are reparameterized version of $w_i^\prime$ and $b^\prime$, respectively, where $\tau$ and $R$ are absorbed. The discrete step size is $1$, so $\tau>1$ is required.

\subsection{Backpropagation Through Time with Surrogate Gradient}
Consider the multi-layer feedforward SNNs with the LIF neurons based on \cref{eqn:lif}:
\begin{equation} \label{eqn:feedforward}
\mathbf{u}^{l}[t]=(1-\frac{1}{\tau})(\mathbf{u}^{l}[t-1] -V_{t h} \mathbf{s}^{l}[t-1])+ \mathbf{W}^{l} \mathbf{s}^{l-1}[t],
\end{equation}
where $l=1,2,\cdots,L$ is the layer index, $t=1,2,\cdots,T$,  $0<1-\frac{1}{\tau}<1$, $\mathbf{s}^0$ are the input data to the network, $\mathbf{s}^l$ are the output spike trains of the $l^{\text{th}}$ layer, $\mathbf{W}^{l}$ are the weight to be trained. We ignore the bias term for simplicity. The final output of the network is $\mathbf{o}[t]=\mathbf{W}^o\mathbf{s}^{L}[t]$, where $\mathbf{W}^o$ is the parameter of the classifier.
The classification is based on the average of the output at each time step $\frac{1}{T}\sum_{t=1}^{T} \mathbf{o}[t]$. 
The loss function $\mathcal{L}$ is defined on $\{\mathbf{o}[1],\cdots,\mathbf{o}[T]\}$, and is often defined as \cite{zheng2020going,rathi2020diet,xiao2021ide,li2021differentiable}
\begin{equation} \label{eqn:loss_rate}
\mathcal{L} = \ell(\frac{1}{T}\sum_{t=1}^{T}\mathbf{o}[t],y),
\end{equation}
where $y$ is the label, and $\ell$ can be the cross-entropy function.

\begin{figure*}[h]
	\begin{subfigure}{0.33\linewidth}
		\includegraphics[width=0.99\columnwidth]{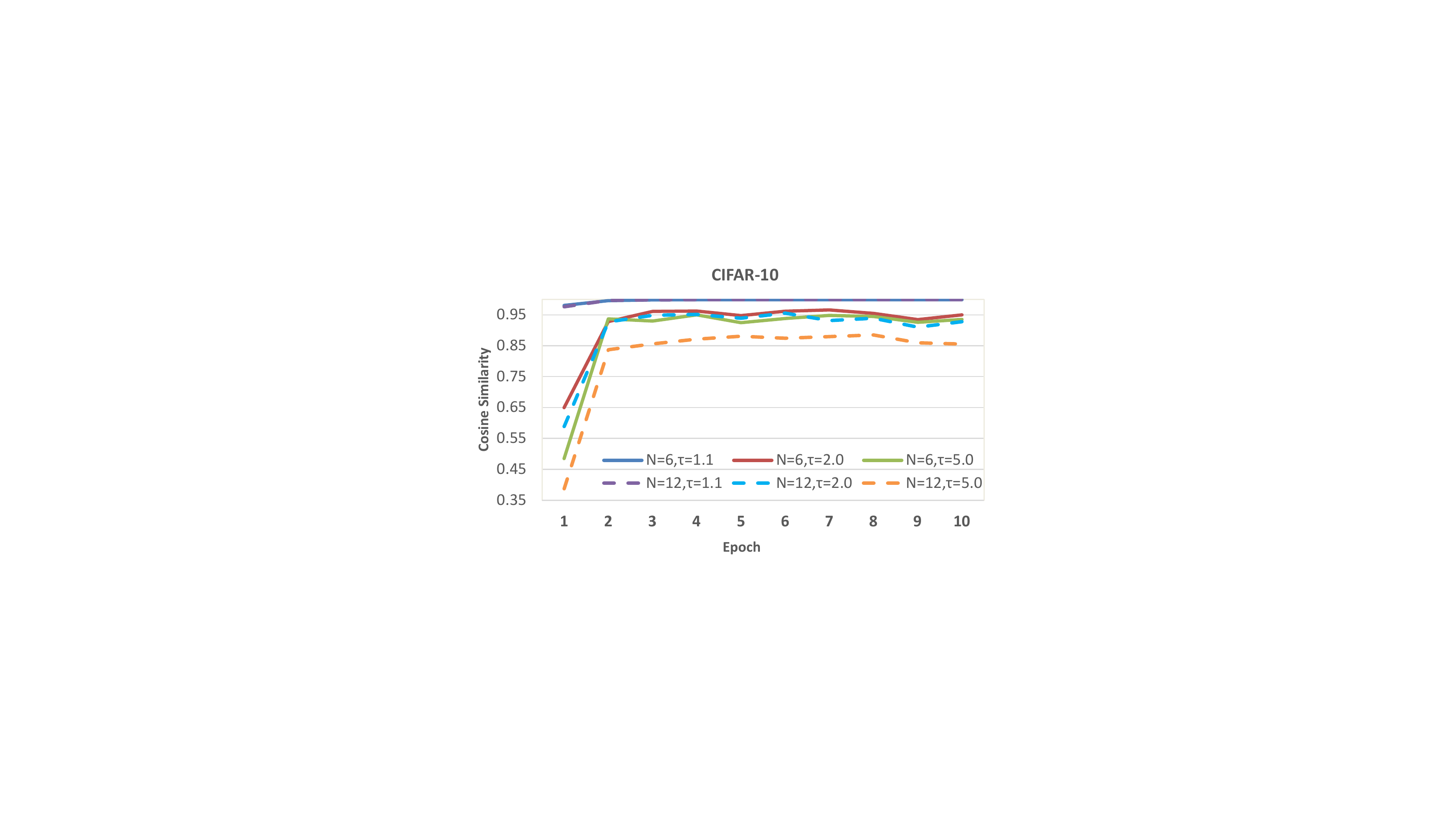} 
	\end{subfigure}
	\begin{subfigure}{0.33\linewidth}
		\includegraphics[width=0.99\columnwidth]{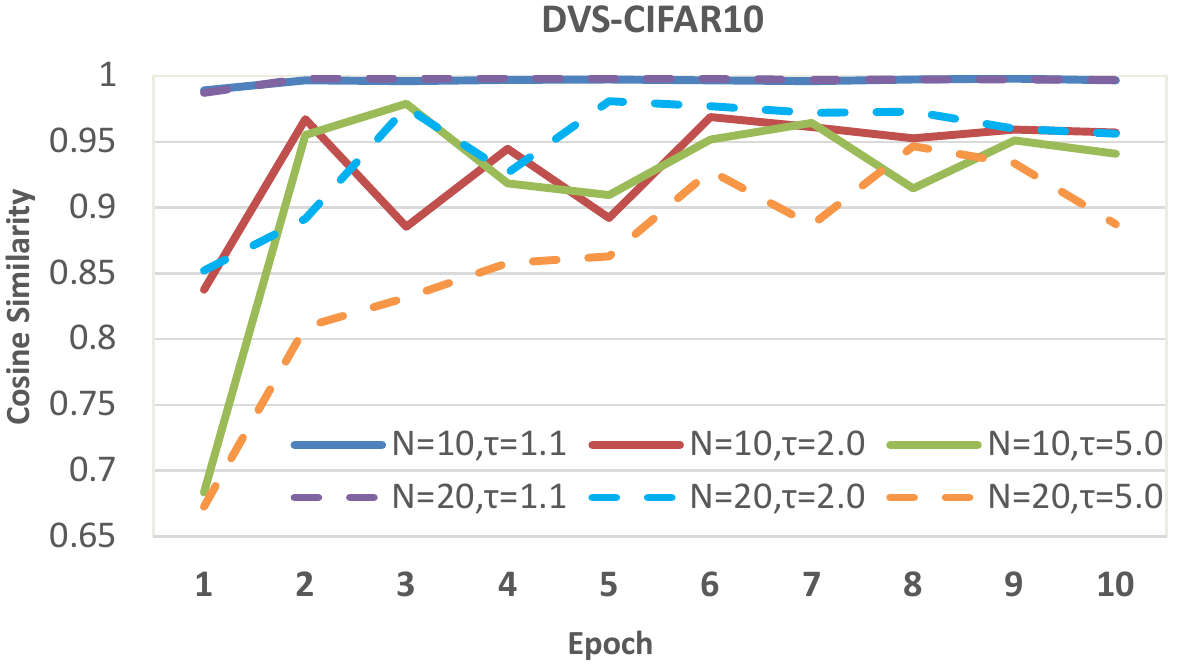} 
	\end{subfigure}
	\begin{subfigure}{0.33\linewidth}
		\includegraphics[width=0.99\columnwidth]{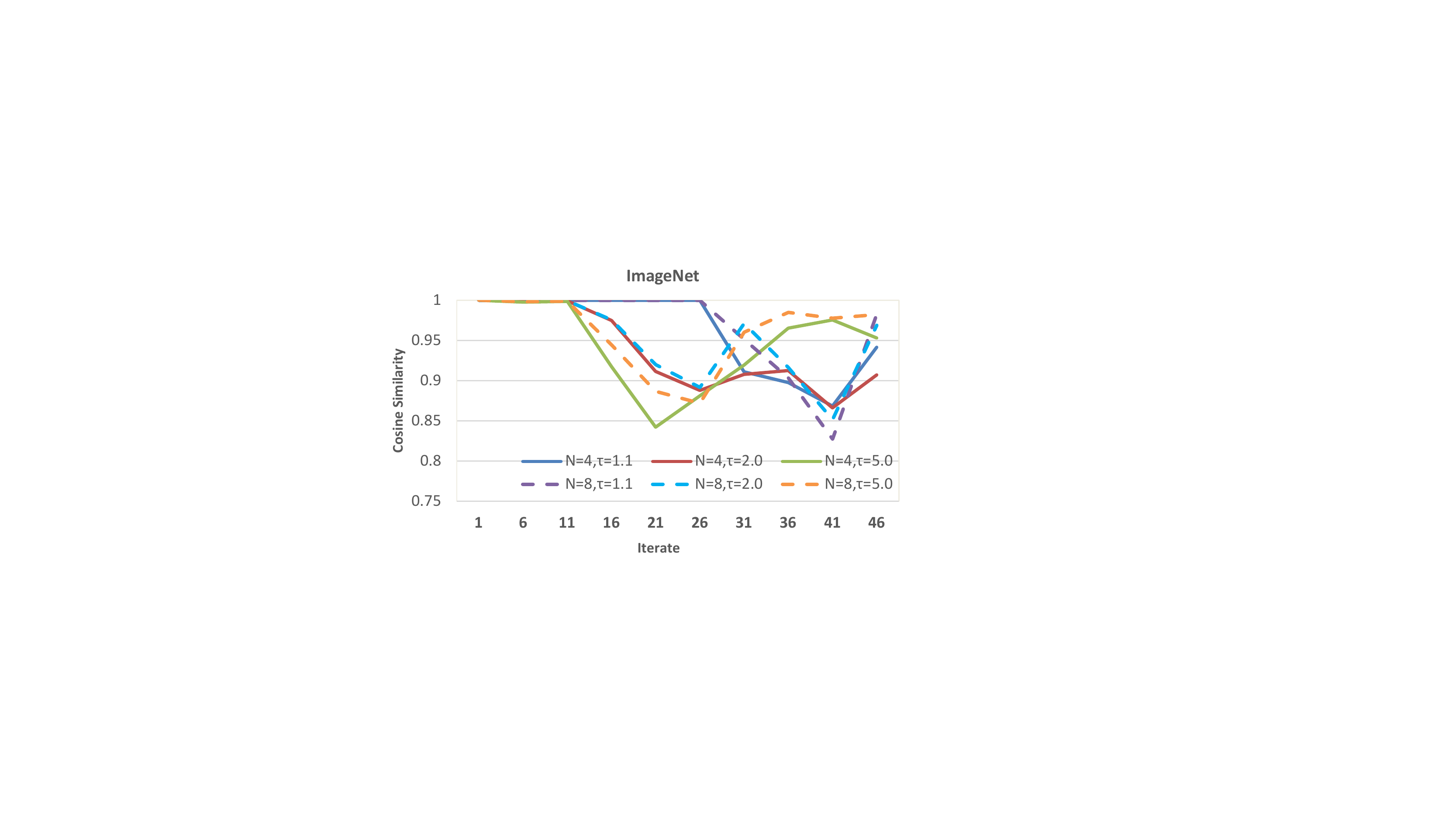}
	\end{subfigure}
	\caption{\small The cosine similarity between the gradients calculated by BPTT and the ``spatial gradients''. For the CIFAR-10, DVS-CIFAR10, and ImageNet datasets, the network architectures of ResNet-18, VGG-11, and ResNet-34 are adopted, respectively. Other settings and hyperparameters for the experiments are described in the Supplementary Materials. We calculate the
		cosine similarity for different layers and report the average in the figure. For ImageNet, we only train the network for 50 iterates since the training is time-consuming. Dashed curves represent a larger number of time steps.}
	\label{fig:gradient_compare}
\end{figure*}

BPTT with SG calculates gradients according to the computational graph of \cref{eqn:feedforward} shown in \cref{fig:bptt}. The pseudocode is described in the Supplementary Materials. For each neuron $i$ in the $l$-th layer, the derivative $\frac{\partial \mathbf{s}_i^{l}[t] }{\partial \mathbf{u}_i^{l}[t]}$ is zero for all values of $\mathbf{u}_i^{l}[t]$ except when $\mathbf{u}_i^{l}[t]=V_{th}$, where the derivative is infinity. Such a non-differentiability problem is solved by approximating $\frac{\partial \mathbf{s}_i^{l}[t] }{\partial \mathbf{u}_i^{l}[t]}$ with some well-behaved surrogate function, such as the rectangle function \cite{wu2018spatio,wu2019direct}
\begin{equation} \label{eqn:rectangle_sg}
\frac{\partial s }{\partial u}=\frac{1}{\gamma} \mathbbm{1}\left(\left|u-V_{th}\right|<\frac{\gamma}{2}\right), 
\end{equation}
and the triangle function \cite{deng2022temporal,EsserMACAABMMBN16}
\begin{equation} \label{eqn:triangle_sg}
\frac{\partial s}{\partial u}=\frac{1}{\gamma^2} \max \left(0, \gamma-\left|u-V_{t h}\right|\right),
\end{equation}
where $\mathbbm{1}(\cdot)$ is the indicator function, and the hyperparameter $\gamma$ for both functions is often set as $V_{th}$.

\section{The proposed Spatial Learning Through Time Method}
\label{sec:SLTT}
\subsection{Observation from the BPTT with SG Method}
\label{sec:observation} 
In this subsection, we decompose the derivatives for membrane potential, as calculated in the BPTT method, into spatial components and temporal components. 
Based on the decomposition, we observe that the spatial components dominate the calculated derivatives. This phenomenon inspires the proposed method, as introduced in \cref{sec:method}.

According to \cref{eqn:feedforward} and \cref{fig:bptt}, the gradients for weights in an SNN with $T$ time steps are calculated by 
\begin{equation} \label{eqn:w-update}
\nabla_{\mathbf{W}^{l}}\mathcal{L}
=\sum_{t=1}^{T}
\frac{\partial \mathcal{L}}{\partial \mathbf{u}^{l}[t]} ^\top
\mathbf{s}^{l-1}[t]^\top, \ l = L, L-1,\cdots,1.
\end{equation}
We further define 
\begin{equation}
\textcolor{black}{\mathbf{\epsilon}^{l}[t]} \triangleq
\frac{\partial \mathbf{u}^{l}[t+1]}{\partial \mathbf{u}^{l}[t]}
+\frac{\partial \mathbf{u}^{l}[t+1]}{\partial \mathbf{s}^{l}[t]}
\frac{\partial \mathbf{s}^{l}[t]}{\partial \mathbf{u}^{l}[t]}
\end{equation}
as the sensitivity of $\mathbf{u}^{l}[t+1]$ with respect to $\mathbf{u}^{l}[t]$, represented by the red arrows shown in \cref{fig:bptt}. Then with the chain rule, $\frac{\partial \mathcal{L}}{\partial \mathbf{u}^{l}[t]}$ in \cref{eqn:w-update} can be further calculated recursively. In particular, for the output layer, we arrive at
\begin{equation} \small \label{eqn:u-update-final}
\frac{\partial \mathcal{L}}{\partial \mathbf{u}^{L}[t]}
=\textcolor{newblue}{
	\frac{\partial \mathcal{L}}{\partial \mathbf{s}^{L}[t]}
	\frac{\partial \mathbf{s}^{L}[t]}{\partial \mathbf{u}^{L}[t]}
}
+ \textcolor{newgreen}{
	\sum_{t^\prime=t+1}^{T}
	\frac{\partial \mathcal{L}}{\partial \mathbf{s}^{L}[t^\prime]}
	\frac{\partial \mathbf{s}^{L}[t^\prime]}{\partial \mathbf{u}^{L}[t^\prime]}
	\prod_{t^{\dprime}=1}^{t^\prime - t}
	\mathbf{\epsilon}^{L}[t^\prime-t^\dprime]
},
\end{equation}
and for the intermediate layer $l=L-1,\cdots,1$, we have
\begin{equation} \small \label{eqn:u-update-inter}
\begin{aligned}
\frac{\partial \mathcal{L}}{\partial \mathbf{u}^{l}[t]}
=&\textcolor{newblue}{
	\frac{\partial \mathcal{L}}{\partial \mathbf{u}^{l+1}[t]}
	\frac{\partial \mathbf{u}^{l+1}[t]}{\partial \mathbf{s}^{l}[t]}
	\frac{\partial \mathbf{s}^{l}[t]}{\partial \mathbf{u}^{l}[t]}}
\\
&+ \textcolor{newgreen}{
	\sum_{t^\prime=t+1}^{T}
	\frac{\partial \mathcal{L}}{\partial \mathbf{u}^{l+1}[t^\prime]}
	\frac{\partial \mathbf{u}^{l+1}[t^\prime]}{\partial \mathbf{s}^{l}[t^\prime]}
	\frac{\partial \mathbf{s}^{l}[t^\prime]}{\partial \mathbf{u}^{l}[t^\prime]}
	\prod_{t^{\dprime}=1}^{t^\prime - t}
	\mathbf{\epsilon}^{l}[t^\prime-t^\dprime]
}.
\end{aligned}
\end{equation}
The detailed derivation can be found in the Supplementary Materials. In both \cref{eqn:u-update-final,eqn:u-update-inter}, the terms before the addition symbols on the R.H.S. (the blue terms) can be treated as the spatial components, and the remaining parts (the green terms) represent the temporal components. 

We observe that the temporal components contribute a little to $\frac{\partial \mathcal{L}}{\partial \mathbf{u}^{l}[t]}$, since the diagonal matrix $\prod_{t^{\dprime}=1}^{t^\prime - t} \textcolor{black}{\mathbf{\epsilon}^{l}[t^\prime-t^\dprime]}$ is supposed to have a small spectral norm for typical settings of surrogate functions.
To see this, we consider the rectangle surrogate (\cref{eqn:rectangle_sg}) with $\gamma=V_{th}$ as an example. Based on \cref{eqn:feedforward}, the diagonal elements of $\textcolor{black}{\mathbf{\epsilon}^{l}[t]}$ are
\begin{equation} \small \label{eqn:small_sensitivity}
\textcolor{black}{\left(\mathbf{\epsilon}^{l}[t]\right)_{jj}} = \left\{\begin{array}{l}0, \quad \frac{1}{2}V_{th}<\left(\mathbf{u}^{l}[t]\right)_j<\frac{3}{2}V_{th}, \\ 1-\frac{1}{\tau}, \quad \text{otherwise}.\end{array}\right.
\end{equation}
Define $\lambda \triangleq 1-\frac{1}{\tau}$, then $\textcolor{black}{\left(\mathbf{\epsilon}^{l}[t]\right)_{jj}}$ is zero in an easily-reached interval, and is at least not large for commonly used small $\lambda$ (\eg, $\lambda=0.5$ \cite{xiao2022online,deng2022temporal}, $\lambda=0.25$ \cite{zheng2020going}, and $\lambda=0.2$ \cite{guo2022recdis}). The diagonal values of the matrix $\prod_{t^{\dprime}=1}^{t^\prime - t} \textcolor{black}{\mathbf{\epsilon}^{l}[t^\prime-t^\dprime]}$ are smaller than the single term $\textcolor{black}{\mathbf{\epsilon}^{l}[t^\prime-t^\dprime]}$ due to the product operations, especially when $t^\prime-t$ is large. The temporal components are further unimportant if the spatial and temporal components have similar directions. Then the spatial components in \cref{eqn:u-update-final,eqn:u-update-inter} dominate the gradients.

For other widely-used surrogate functions and their corresponding hyperparameters, the phenomenon of dominant spatial components still exists since the surrogate functions have similar shapes and behavior. In order to illustrate this, we conduct experiments on CIFAR-10, DVS-CIFAR10, and ImageNet using the triangle surrogate (\cref{eqn:triangle_sg}) with $\gamma=V_{th}$. We use the BPTT with SG method to train the SNNs on the abovementioned three datasets, and call the calculated gradients the baseline gradients. During training, we also calculate the gradients for weights when the temporal components are abandoned, and call such gradients the spatial gradients. We compare the disparity between baseline and spatial gradients by calculating their cosine similarity. The results are demonstrated in \cref{fig:gradient_compare}. The similarity maintains a high level for different datasets, the number of time steps, and $\tau$. In particular, for $\tau=1.1 \ (\lambda=1-\frac{1}{\tau}\approx0.09)$, the baseline and spatial
gradients consistently have a remarkably similar direction on CIFAR-10 and DVS-CIFAR10. In conclusion, the spatial components play a dominant role in the gradient backpropagation process.

\subsection{Spatial Learning Through Time}
\label{sec:method}

Based on the observation introduced in \cref{sec:observation}, we propose to ignore the temporal components in \cref{eqn:u-update-final,eqn:u-update-inter} to achieve more efficient backpropagation. In detail, the gradients for weights are calculated by
\begin{equation} \label{eqn:w-update-ours}
\nabla_{\mathbf{W}^{l}}\mathcal{L}
=\sum_{t=1}^{T} \mathbf{e}_\mathbf{W}^{l}[t], \quad
\mathbf{e}_\mathbf{W}^{l}[t]  = \mathbf{e}_\mathbf{u}^{l}[t] ^\top
\mathbf{s}^{l-1}[t]^\top,
\end{equation}
where 
\begin{equation} \label{eqn:u-update-ours}
\mathbf{e}_\mathbf{u}^{l}[t] = \left\{\begin{array}{l}
\frac{\partial \mathcal{L}}{\partial \mathbf{s}^{L}[t]}
\frac{\partial \mathbf{s}^{L}[t]}{\partial \mathbf{u}^{L}[t]}, \quad\quad\quad\quad\quad \ l=L, \\ 
\mathbf{e}_\mathbf{u}^{l+1}[t]
\frac{\partial \mathbf{u}^{l+1}[t]}{\partial \mathbf{s}^{l}[t]}
\frac{\partial \mathbf{s}^{l}[t]}{\partial \mathbf{u}^{l}[t]}, 
\quad \quad l<L,\end{array}\right.
\end{equation}
and $\mathbf{e}_\mathbf{u}^{l}[t]$ is a row vector.
Compared with \cref{eqn:w-update,eqn:u-update-final,eqn:u-update-inter}, the required number of scalar multiplications in \cref{eqn:w-update-ours,eqn:u-update-ours} is reduced from $\Omega(T^2)$ to $\Omega(T)$. 
Note that the BPTT method does not conduct naive
computation of the sum-product as shown in \cref{eqn:u-update-final,eqn:u-update-inter}, but in a recursive way to achieve $\Omega(T)$ computational complexity, as shown in the Supplementary Materials. Although BPTT and the proposed update rule both need $\Omega(T)$ scalar multiplications, such multiplication operations are reduced due to ignoring some routes in the computational graph. Please refer to Supplementary Materials for time complexity analysis.
Therefore, the time complexity of the proposed update rule is much lower than that of BPTT with SG, although they are both proportional to $T$.

According to \cref{eqn:w-update-ours,eqn:u-update-ours}, the error signals $\mathbf{e}_\mathbf{W}^{l}$ and $\mathbf{e}_\mathbf{u}^{l}$ at each time step can be calculated independently without information from other time steps. Thus, if $\frac{\partial \mathcal{L}}{\partial \mathbf{s}^{L}[t]}$ can be calculated instantaneously at time step $t$, $\mathbf{e}_\mathbf{W}^{l}[t]$ and $\mathbf{e}_\mathbf{u}^{l}[t]$ can also be calculated instantaneously at time step $t$. Then there is no need to store intermediate states of the whole time horizon. To achieve the instantaneous calculation of $\frac{\partial \mathcal{L}}{\partial \mathbf{s}^{L}[t]}$,  we adopt the loss function 
\cite{guo2022recdis,deng2022temporal,xiao2022online}
\begin{equation} \label{eqn:loss}
\mathcal{L} = \frac{1}{T}\sum_{t=1}^{T}\ell(\mathbf{o}[t],y),
\end{equation}
which is an upper bound of the loss introduced in \cref{eqn:loss_rate}.

\vspace{5pt}

\setlength{\textfloatsep}{14pt}
\begin{algorithm}[h] 
	\caption{One iteration of SNN training with the SLTT or SLTT-K methods.}
	\label{alg:SLTT}
	\begin{algorithmic}[1] 
		\Require Time steps $T$; Network depth $L$; Network parameters $\{\mathbf{W}^l\}_{l=1}^L$; Training data $(\mathbf{s}^0,\mathbf{y})$; Learning rate $\eta$; Required backpropagation times $K$ (for SLTT-K).
		\item[\textbf{Initialize:}] $\Delta \mathbf{W}^l = 0, \ l=1,2,\cdots,L$.
		\If {using SLTT-K}
		\State Sample $K$ numbers in $[1,2,\cdots,T]$ w/o replacement to form $required\_bp\_steps$;
		\Else
		\State $required\_bp\_steps=[1,2,\cdots,T]$;
		\EndIf
		
		\For {$t=1,2,\cdots,T$}
		\State Calculate $\mathbf{s}^L[t]$ by \cref{eqn:feedforward,eqn:lif}; \quad //\textbf{Forward}
		\State Calculate the instantaneous loss $\ell$ in \cref{eqn:loss};
		\If {$t$ in $required\_bp\_steps$} \quad\quad //\textbf{Backward}
		\State $\mathbf{e}_\mathbf{u}^{L}[t] =\frac{1}{T}\frac{\partial \ell}{\partial \mathbf{s}^{L}[t]}\frac{\partial \mathbf{s}^{L}[t]}{\partial \mathbf{u}^{L}[t]}$;
		\For {$l=L-1,\cdots,1$}
		\State $\mathbf{e}_\mathbf{u}^l[t] = \mathbf{e}_\mathbf{u}^{l+1}[t]
		\frac{\partial \mathbf{u}^{l+1}[t]}{\partial \mathbf{s}^{l}[t]}
		\frac{\partial \mathbf{s}^{l}[t]}{\partial \mathbf{u}^{l}[t]}$;
		\State $\Delta \mathbf{W}^l \mathrel{+}= \mathbf{e}_\mathbf{u}^{l}[t] ^\top \mathbf{s}^{l-1}[t]^\top$;
		\EndFor
		\EndIf
		\EndFor
		
		\State $\mathbf{W}^l = \mathbf{W}^l - \eta \Delta \mathbf{W}^l, \ l=1,2,\cdots,L$;
		
		\Ensure Trained network parameters $\{\mathbf{W}^l\}_{l=1}^L$.
	\end{algorithmic}
\end{algorithm} 

We propose the Spatial Learning Through Time (SLTT) method using gradient approximation and instantaneous gradient calculation, as detailed in \cref{alg:SLTT}. In \cref{alg:SLTT}, all the intermediate terms at time step $t$, such as $\mathbf{e}_\mathbf{u}^{l}[t], \mathbf{s}^{l}[t],\frac{\partial \mathbf{u}^{l+1}[t]}{\partial \mathbf{s}^{l}[t]}$, and $\frac{\partial \mathbf{s}^{l}[t]}{\partial \mathbf{u}^{l}[t]}$, are never used in other time steps, so the required memory overhead of SLTT is constant agnostic to the total number of time steps $T$. On the contrary, the BPTT with SG method has an $\Omega(T)$ memory cost associated with storing all intermediate states for all time steps. In summary, the proposed method is both time-efficient and memory-efficient, and has the potential to enable online learning for neuromorphic substrates \cite{zenke2021brain}.

\subsection{Further Reducing Time Complexity}
\label{sec:rethink}

Due to the online update rule of the proposed method, the gradients for weights are calculated according to an ensemble of $T$ independent computational graphs, and the time complexity of gradient calculation is $\Omega(T)$. 
The $T$ computational graphs can have similar behavior, and then similar gradient directions can be obtained with only a portion of the computational graphs.
Based on this, we propose to train a portion of time steps to reduce the time complexity further. In detail, for each iteration in the training process, we randomly choose $K$ time indexes from the time horizon, and only conduct backpropagation with SLTT at the chosen $K$ time steps. We call such a method the SLTT-K method, and the pseudo-code is given in \cref{alg:SLTT}. Note that setting $K=T$ results in the original SLTT method. Compared with SLTT, the time complexity of SLTT-K is reduced to $\Omega(K)$, and the memory complexity is the same. In our experiments, SLTT-K can achieve satisfactory performance even when $K=1$ or $2$, as shown in \cref{sec:experiments}, indicating superior efficiency of the SLTT-K method.

\section{Experiments}
\label{sec:experiments}
In this section, we evaluate the proposed method on CIFAR-10 \cite{krizhevsky2009learning}, CIFAR-100\cite{krizhevsky2009learning}, ImageNet\cite{deng2009imagenet}, DVS-Gesture\cite{amir2017low}, and DVS-CIFAR10 \cite{li2017cifar10} to demonstrate its superior performance regarding training costs and accuracy. For our SNN models, we set $V_{th}=1$ and $\tau=1.1$, and apply the triangle surrogate function (\cref{eqn:triangle_sg}).
An effective technique, batch normalization (BN) along the temporal dimension \cite{zheng2020going}, cannot be adopted to our method, since it requires calculation along the total time steps and then intrinsically prevents time-steps-independent memory costs. Therefore, for some tasks, we borrow the idea from normalization-free ResNets (NF-ResNets) \cite{brock2021high} to replace BN by weight standardization (WS) \cite{qiao2019micro}.
Please refer to the Supplementary Materials for experimental details.

\begin{table}[t] 
	\caption{\small Comparison of training memory cost, training time, and accuracy between SLTT and BPTT. The ``Memory'' column indicates the maximum memory usage on an GPU during training. And the ``Time'' column indicates the wall-clock training time.}
	\label{table:compare}
	\begin{threeparttable}
		\begin{tabular}{lcccc}
			\toprule  
			Dataset  & Method & Memory & Time & Acc \\
			\midrule 
			\multirow{2}*{CIFAR-10}&  BPTT & 3.00G & 6.35h & \bf{94.60\%}\\
			&  SLTT & \bf{1.09G} & \bf{4.58h} & 94.59\%  \\
			\hline
			\multirow{2}*{CIFAR-100}  &  BPTT & 3.00G & 6.39h & 73.80\% \\
			&  SLTT & \bf{1.12}G & \bf{4.68h} & \bf{74.67}\% \\
			\hline
			\multirow{2}*{ImageNet}  &  BPTT & 28.41G & 73.8h & \bf{66.47\%} \\
			&  SLTT & \bf{8.47G} & \bf{66.9h} & 66.19\% \\
			\hline
			\multirow{2}*{DVS-Gesture} &  BPTT & 5.82G & 2.68h & 97.22\% \\
			& SLTT & \bf{1.07G} & \bf{2.64h} & \bf{97.92\%} \\
			\hline
			\multirow{2}*{DVS-CIFAR10} &  BPTT & 3.70G & 4.47h & 73.60\% \\
			&  SLTT & \bf{1.07G} & \bf{3.43h}  & \bf{77.30\%} \\
			\bottomrule
		\end{tabular}
	\end{threeparttable}
\end{table}

\subsection{Comparison with BPTT}
\label{sec:bptt_compare}

The major advantage of SLTT over BPTT is the low memory and time complexity. To verify the advantage of SLTT, we use both methods with the same experimental setup to train SNNs. For CIFAR-10, CIFAR-100, ImageNet, DVS-Gesture, and DVS-CIFAR10, the network architectures we adopt are ResNet-18, ResNet-18, NF-ResNet-34, VGG-11, and VGG-11, respectively, and the total number of time steps are 6, 6, 6, 20, and 10, respectively. For ImageNet, to accelerate training, we first train the SNN with only 1 time step for 100 epochs to get a pre-trained model, and then use SLTT or BPTT to fine-tune the model with 6 time steps for 30 epochs. Details of the training settings can be found in the Supplementary Materials. We run all the experiments on the same Tesla-V100 GPU, and ensure that the GPU card is running only one experiment at a time to perform a fair comparison. It is not easy to directly compare the running time for two training methods since the running time is code-dependent and platform-dependent. In our experiments, we measure the wall-clock time of the total training process, including forward propagation and evaluation on the validation set after each epoch, to give a rough comparison. For ImageNet, the training time only includes the 30-epoch fine-tuning part.

\begin{table}[h] 
	\caption{\small Comparison of training time and accuracy between SLTT and SLTT-K. ``NFRN'' means Normalizer-Free ResNet. For DVS-Gesture and DVS-CIFAR10, the ``Acc'' column reports the average accuracy of 3 runs of experiments using different random seeds. We skip the standard deviation values since they are almost 0, except for SLTT on DVS-CIFAR10 where the value is 0.23\%.}
	\label{table:SLTT-k}
	\begin{threeparttable}
		\begin{tabular}{lcccc}
			\toprule  
			Network & Method  & Memory & Time & Acc \\
			\midrule 
			\multicolumn{5}{c}{{DVS-Gesture, $T=20$}}\\
			\midrule 
			\multirow{2}*{VGG-11}  &  SLTT & \multirow{2}*{$\approx$1.1G} & 2.64h
			&  \bf{97.92\%} \\
			&  SLTT-4 &  & \bf{1.69h} & 97.45\% \\
			\midrule
			\multicolumn{5}{c}{{DVS-CIFAR10, $T=10$}}\\
			\midrule 
			\multirow{2}*{VGG-11}  &  SLTT & \multirow{2}*{$\approx$1.1G} & 3.43h
			& \bf{77.16\%} \\
			&  SLTT-2 &  & \bf{2.49h} & 76.70\% \\
			\midrule
			\multicolumn{5}{c}{{ImageNet, $T=6$}}\\
			\midrule
			\multirow{3}*{NFRN-34}  &  SLTT & \multirow{3}*{$\approx$8.5G} & 66.90h
			& \bf{66.19\%} \\
			&  SLTT-2 &  & 41.88h & 66.09\% \\
			&  SLTT-1 &  & \bf{32.03h} & 66.17\% \\
			\hline
			\multirow{3}*{NFRN-50} & SLTT & \multirow{3}*{$\approx$24.5G} & 126.05h & \bf{67.02\%} \\
			&  SLTT-2 &  & 80.63h & 66.98\% \\
			&  SLTT-1 &  & \bf{69.36h} & 66.94\% \\
			\hline
			\multirow{3}*{NFRN-101} &  SLTT & \multirow{3}*{$\approx$33.8G} & 248.23h & 69.14\% \\
			&  SLTT-2 &  & 123.05h & \bf{69.26\%} \\
			&  SLTT-1 &  & \bf{91.73h} & 69.14\% \\ 
			\bottomrule
		\end{tabular}
	\end{threeparttable}
\end{table}

The results of maximum memory usage, total wall-clock training time, and accuracy for both SLTT and BPTT on different datasets are listed in \cref{table:compare}. SLTT enjoys similar accuracy compared with BPTT while using less memory and time. For all the datasets, SLTT requires less than one-third of the GPU memory of BPTT. In fact, SLTT maintains constant memory cost over the different number of time steps $T$, while the training memory of BPTT grows linearly in $T$. The memory occupied by SLTT for $T$ time steps is always similar to that of BPTT for $1$ time step. Regarding training time, SLTT also enjoys faster training on both algorithmic and practical aspects. For DVS-Gesture, the training time for both methods are almost the same, deviating from the algorithmic time complexity. That may be due to really little training time for both methods and the good parallel computing performance of the GPU.

\subsection{Performance of SLTT-K}

As introduced in \cref{sec:rethink}, the proposed SLTT method has a variant, SLTT-K, that conducts backpropagation only in randomly selected $K$ time steps for reducing training time. We verify the effectiveness of SLTT-K on the neuromorphic datasets, DVS-Gesture and DVS-CIFAR10, and the large-scale static dataset, ImageNet. For the ImageNet dataset, we first pre-train the 1-time-step networks, and then fine-tune them with 6 time steps, as described in \cref{sec:bptt_compare}.
We train the NF-ResNet-101 networks on a single Tesla-A100 GPU, while we use a single Tesla-V100 GPU for other experiments.
As shown in \cref{table:SLTT-k}, the SLTT-K method yields competitive accuracy with SLTT (also BPTT) for different datasets and network architectures, even when $K=\frac{1}{6}T$ or $\frac{1}{5}T$.
With such small values of $K$, further compared with BPTT, the SLTT-K method enjoys comparable or even better training results, less memory cost (much less if $T$ is large), and much faster training speed.

\begin{table}[t] 
	\caption{\small Comparison of training memory cost and training time per epoch between SLTT and OTTT.}
	\label{table:compare_ottt}
	\centering
	\begin{threeparttable}
		\begin{tabular}{cccc}
			\toprule  
			Dataset  & Method & Memory & Time/Epoch  \\
			\midrule 
			\multirow{2}*{CIFAR-10}&  OTTT & 1.71G & 184.68s \\
			&  SLTT & \bf{1.00G} & \bf{54.48s}  \\ 
			\hline
			\multirow{2}*{CIFAR-100}  & OTTT & 1.71G & 177.72s \\
			&  SLTT & \bf{1.00}G & \bf{54.60s} \\  
			\hline
			\multirow{2}*{ImageNet}  &  OTTT & 19.38G &  7.52h  \\
			&  SLTT & \bf{8.47G} & \bf{2.23h}  \\
			\hline
			\multirow{2}*{DVS-Gesture} &  OTTT & 3.38G & 236.64s  \\
			& SLTT & \bf{2.08G} & \bf{67.20s} \\
			\hline
			\multirow{2}*{DVS-CIFAR10} &  OTTT & 4.32G & 114.84s  \\
			&  SLTT & \bf{1.90G} & \bf{48.00s}  \\
			\bottomrule
		\end{tabular}
	\end{threeparttable}
\end{table}

\subsection{Comparison with Other Efficient Training Methods}
\label{sec:ottt_compare}

There are other online learning methods for SNNs \cite{xiao2022online,bellec2020solution,bohnstingl2022online,yin2021accurate,yang2022training} that achieve time-steps-independent memory costs. Among them, OTTT \cite{xiao2022online} enables direct training on large-scale datasets with relatively low training costs. In this subsection, we compare SLTT and OTTT under the same experimental settings of network structures and total time steps (see Supplementary Materials for details). The wall-clock training time and memory cost are calculated based on 3 epochs of training.
The two methods are comparable since the implementation of them are both based on PyTorch \cite{paszke2019pytorch} and SpikingJelly \cite{SpikingJelly}. The results are shown in \cref{table:compare_ottt}. SLTT outperforms OTTT on all the datasets regarding memory costs and training time, indicating the superior efficiency of SLTT. As for accuracy, SLTT also achieves better results than OTTT, as shown in \cref{table:sota}.

\subsection{Comparison with the State-of-the-Art}

\begin{table*}[h]
	\caption{\small Comparisons with other SNN training methods on CIFAR-10, CIFAR-100, ImageNet, DVS-Gesture, and DVS-CIFAR10. Results of our method on all the datasets, except ImageNet, are based on 3 runs of experiments. The ``Efficient Training'' column means whether the method requires less training time or memory occupation than the vanilla BPTT method for one epoch of training.}
	\label{table:sota}
	\centering
	\begin{threeparttable}
		\begin{tabular}{c|lcccc}
			\toprule[1.08pt] & Method & Network & Time Steps & Efficient Training  & Mean$\pm$Std (Best) \\
			\midrule[1.08pt]
			
			\multirow{5}*{\rotatebox{90}{CIFAR-10}} 
			&LTL-Online\cite{yang2022training} \tnote{1} & ResNet-20 & 16 & \Checkmark  & ${93.15\%}$ \\
			& OTTT\cite{xiao2022online} & VGG-11 (WS) & 6 & \Checkmark  & $93.52\pm0.06\%$ ($93.58\%$)  \\
			&Dspike\cite{li2021differentiable} & ResNet-18  & 6 & \XSolidBrush & $94.25\pm0.07\%$  \\
			& TET\cite{deng2022temporal} & ResNet-19 & 6  & \XSolidBrush  & $\mathbf{94.50\pm0.07\%}$  \\
			\cline{2-6}
			&SLTT (ours) & ResNet-18 & 6 & \Checkmark & $\underline{94.44\%\pm0.21\%}$ ($\mathbf{94.59\%}$) \\
			\midrule[1.08pt]
			
			\multirow{5}*{\rotatebox{90}{CIFAR-100}} 
			&OTTT\cite{xiao2022online} & VGG-11 (WS)  & 6 & \Checkmark & $71.05\pm0.04\%$ ($71.11\%$)  \\
			&ANN-to-SNN\cite{bu2022optimal} \tnote{1} & VGG-16 & 8 & \Checkmark  & ${73.96}\%$ \\
			&RecDis\cite{guo2022recdis} & ResNet-19 & 4 & \XSolidBrush  & $74.10\pm0.13 \%$ \\
			&TET\cite{deng2022temporal} & ResNet-19 & 6 & \XSolidBrush  & $\mathbf{74.72\pm0.28\%}$  \\
			\cline{2-6}
			&SLTT (ours) & ResNet-18 & 6 & \Checkmark  & $\underline{74.38\%\pm0.30\% \ (74.67\%)}$  \\
			\midrule[1.08pt]
			
			\multirow{6}*{\rotatebox{90}{ImageNet}} 	
			&ANN-to-SNN\cite{li2021free} \tnote{1} & ResNet-34  & 32 & \Checkmark & ${64.54\%}$ \\
			&TET\cite{deng2022temporal} & ResNet-34 & 6 & \XSolidBrush  & $64.79\%$ \\
			&OTTT\cite{xiao2022online} & NF-ResNet-34 & 6 & \Checkmark  & ${65.15\%}$ \\
			&SEW \cite{fang2021sew} & Sew ResNet-34,50,101  & 4 & \XSolidBrush  & $67.04\% ,67.78\%,68.76\%$ \\
			\cline{2-6}
			&SLTT (ours) & NF-ResNet-34,50 & 6 & \Checkmark  & ${66.19\%,67.02\%}$  \\
			&SLTT-2 (ours) & NF-ResNet-101 & 6 & \Checkmark  & $\mathbf{69.26\%}$  \\
			\midrule[1.08pt]
			
			\multirow{6}*{\rotatebox{90}{\small DVS-Gesture}} 
			&STBP-tdBN \cite{zheng2020going} & ResNet-17 & 40 & \XSolidBrush  &  $96.87\%$ \\
			&OTTT \cite{xiao2022online} & VGG-11 (WS) & 20 & \Checkmark  &  $96.88\%$  \\
			&PLIF \cite{fang2021incorporating} & VGG-like  & 20 & \XSolidBrush & $97.57\%$  \\
			& SEW\cite{fang2021sew} & Sew ResNet & 16  & \XSolidBrush  & ${97.92}\%$  \\
			\cline{2-6}
			&\multirow{2}*{SLTT (ours)} & VGG-11 & 20 & \Checkmark  & $\underline{97.92\pm0.00\% \ (97.92\%)}$ \\
			& & VGG-11 (WS) & 20 & \Checkmark  & $\mathbf{98.50\pm0.21\% \ (98.62\%)}$ \\
			\midrule[1.08pt]
			
			\multirow{6}*{\rotatebox{90}{DVS-CIFAR10}} 
			&Dspike\cite{li2021differentiable} & ResNet-18 & 10 & \XSolidBrush  & $75.40\pm0.05\%$  \\
			& InfLoR\cite{guo2022reducing} & ResNet-19 & 10 & \XSolidBrush   & $75.50\pm0.12\%$  \\	
			&OTTT \cite{xiao2022online} & VGG-11 (WS) & 10 & \Checkmark  &  $76.27\pm0.05\% (76.30\%)$  \\
			& TET\cite{deng2022temporal} & VGG-11 & 10 & \XSolidBrush   & $\mathbf{83.17\pm0.15\%}$  \\
			\cline{2-6}
			&SLTT (ours) \tnote{2} & VGG-11  & 10 & \Checkmark & ${77.17\pm0.23\% \ (77.30\%)}$ \\
			&SLTT (ours)  & VGG-11  & 10 & \Checkmark & $\underline{82.20\pm0.95\% \ (83.10\%)}$ \\
			\bottomrule[1.08pt]
		\end{tabular}
		\small
		$^{1}$ Pre-trained ANN models are required. \ $^{2}$ Without data augmentation.
	\end{threeparttable}
\end{table*}

The proposed SLTT method is not designed to achieve the best accuracy, but to enable more efficient training. Still, our method achieves competitive results compared with the SOTA methods, as shown in \cref{table:sota}. Besides, our method obtains such good performance with only a few time steps, leading to low energy consumption when the trained networks are implemented on neuromorphic hardware.

\begin{figure}[t]
	\centering
	\begin{subfigure}{0.49\linewidth}
		\includegraphics[width=1.12\columnwidth]{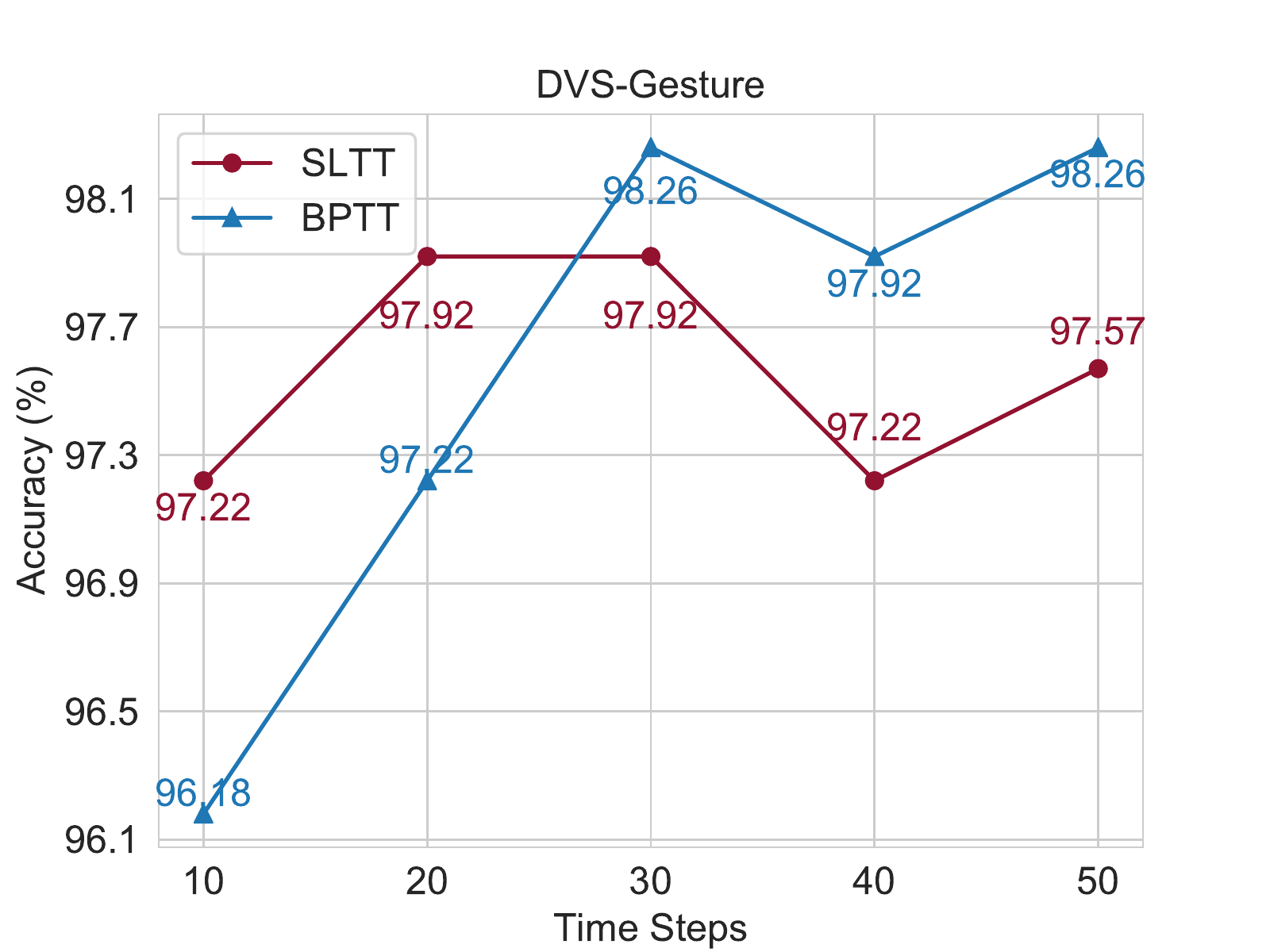} 
		\label{fig:dvsgesture_step}
	\end{subfigure}
	\begin{subfigure}{0.49\linewidth}
		\includegraphics[width=1.12\columnwidth]{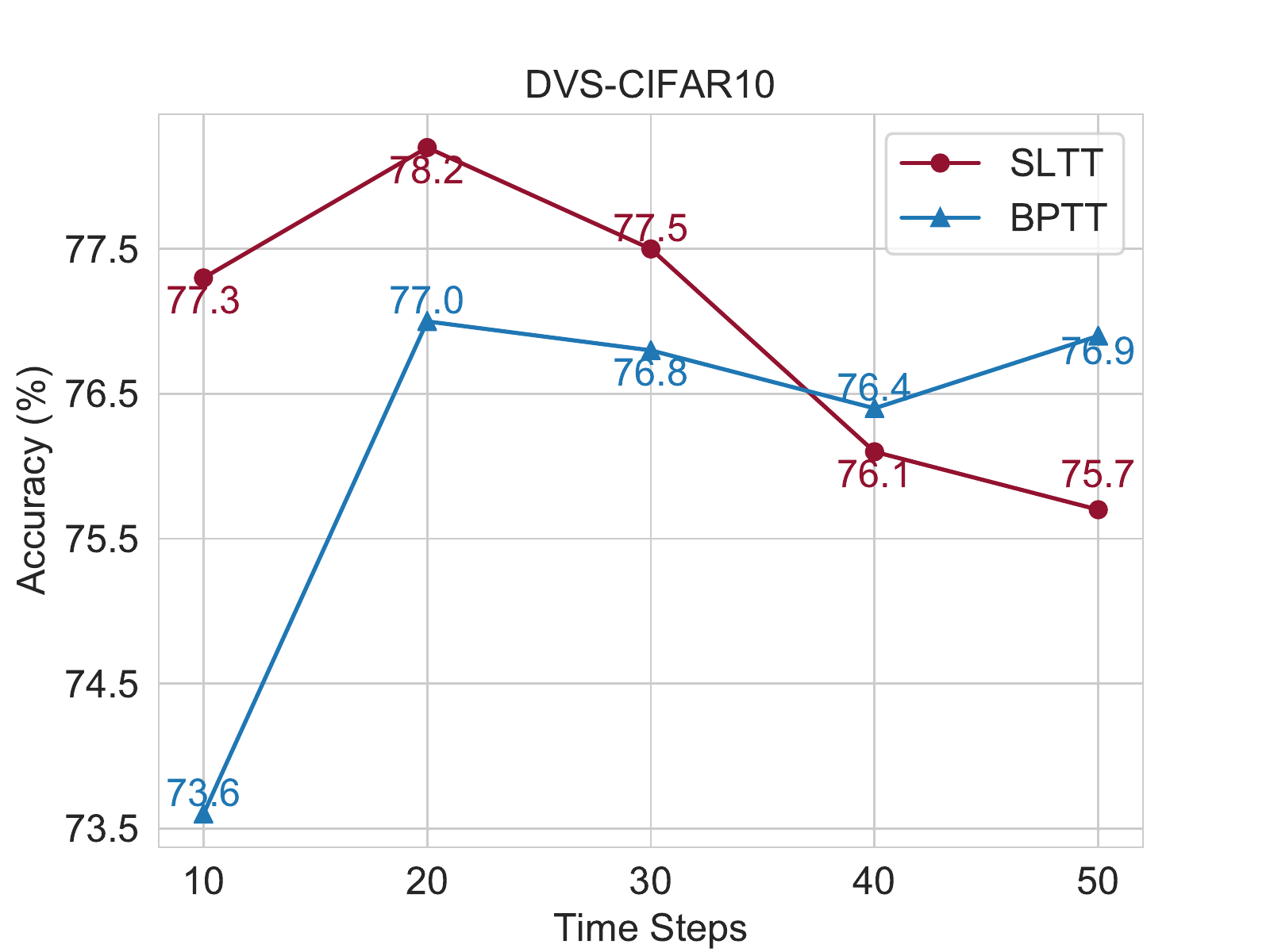} 
		\label{fig:dvscifar_step}
	\end{subfigure}
	\begin{subfigure}{0.49\linewidth}
		\vspace{-1.0em}
		\includegraphics[width=1.12\columnwidth]{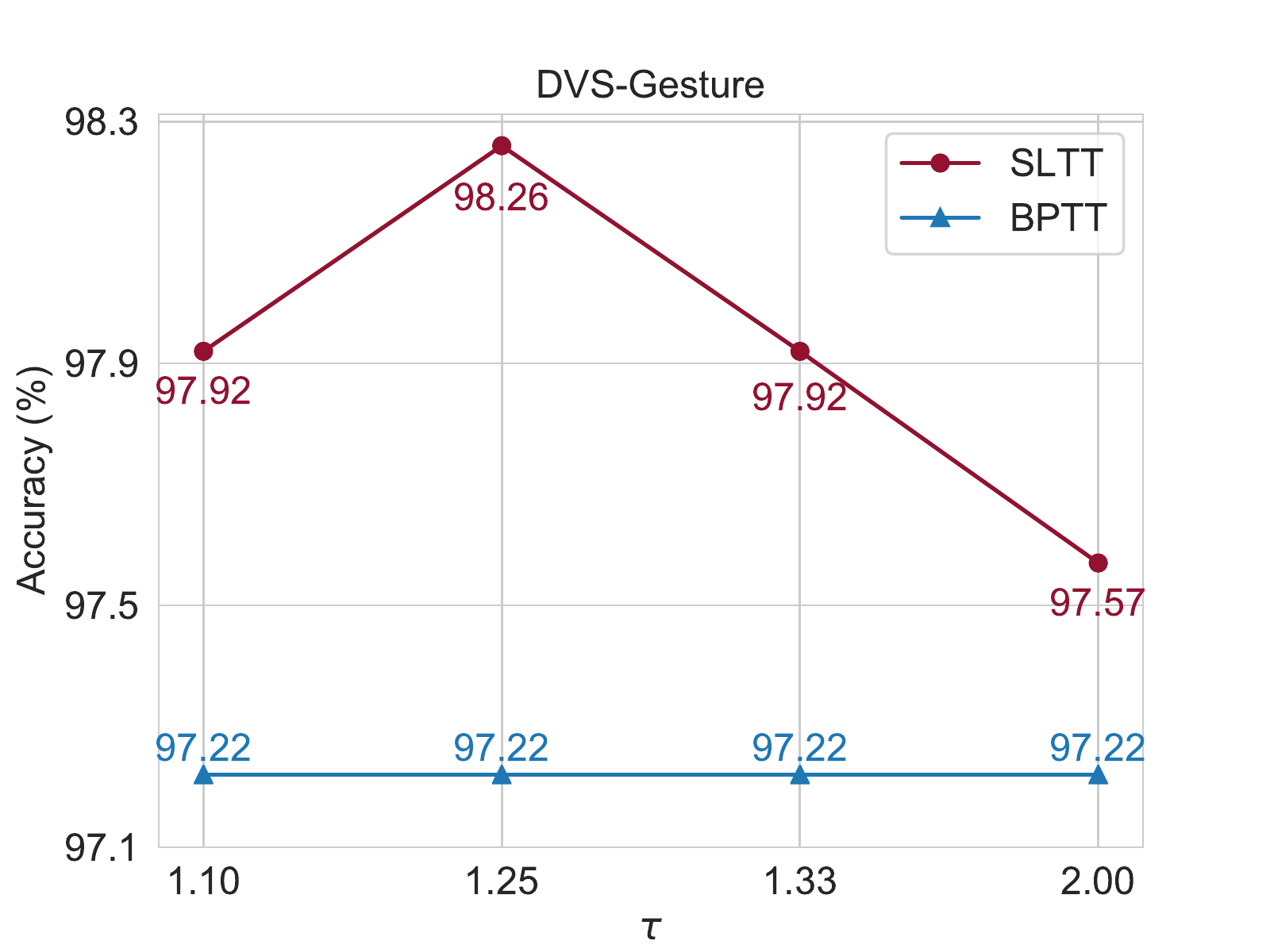} 
		\label{fig:dvsgesture_tau}
	\end{subfigure}
	\begin{subfigure}{0.49\linewidth}
		\vspace{-1.0em}
		\includegraphics[width=1.12\columnwidth]{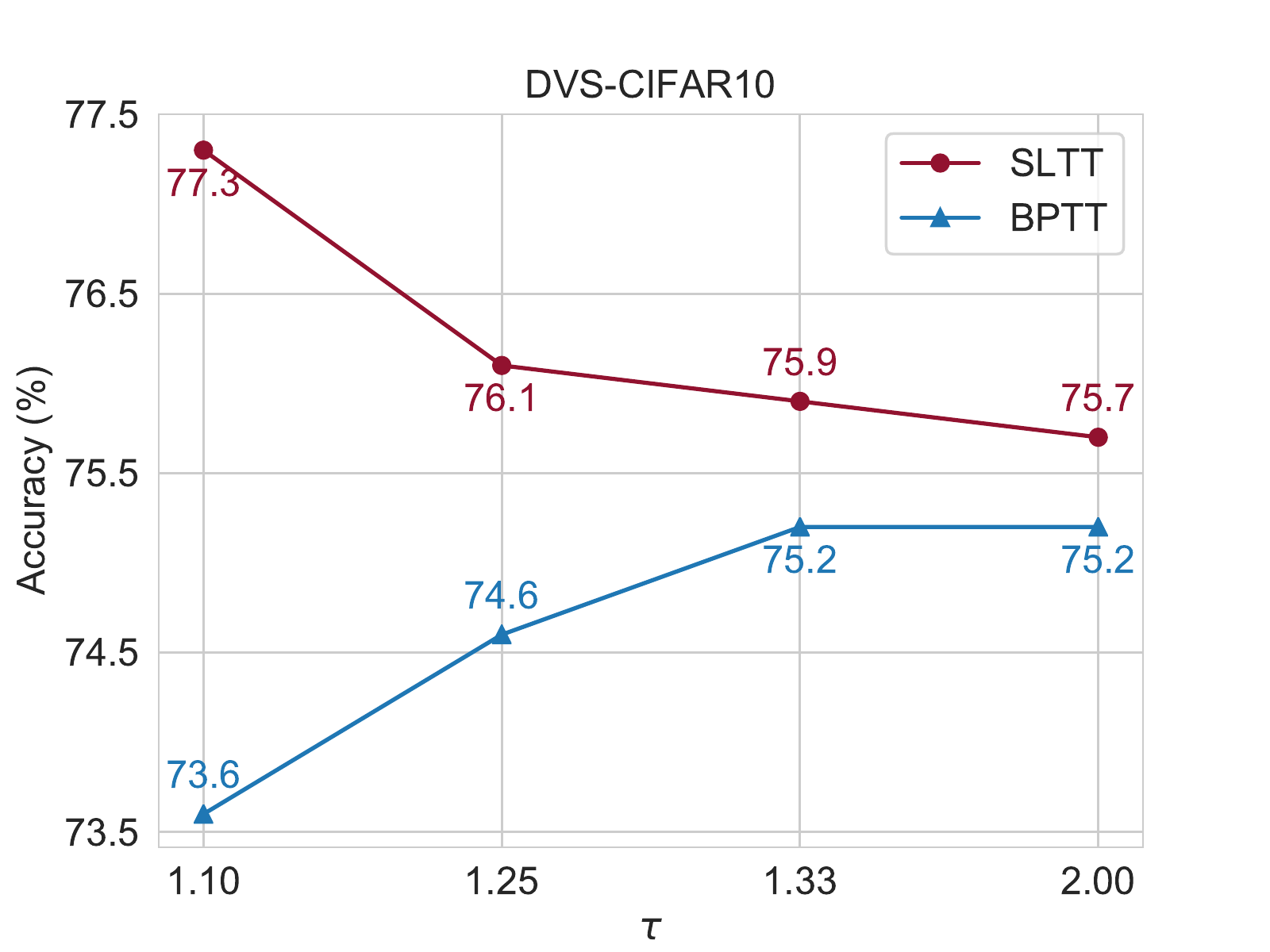} 
		\label{fig:dvscifar_tau}
	\end{subfigure}
	\vspace{-1.3em}
	\caption{\small Performance of SLTT and BPTT for different number of time steps (the top two subfigures) and for different $\tau$ (the bottom two subfigures). Experiments are conducted on the neuromorphic datasets, DVS-Gesture and DVS-CIFAR10.}
	\label{fig:diff_tau_step}
	\vspace{-3pt}
\end{figure}

For the BPTT-based methods, there is hardly any implementation of large-scale network architectures on ImageNet due to the significant training costs. To our knowledge, only Fang \etal \cite{fang2021sew} leverage BPTT to train an SNN with more than 100 layers, while the training process requires near 90G GPU memory for $T=4$. Our SLTT-2 method succeeds in training the same-scale ResNet-101 network with only 34G memory occupation and 4.10h of training time per epoch (\cref{table:SLTT-k,table:sota}). Compared with BPTT, the training memory and time of SLTT-2 are reduced by more than 70\% and 50\%, respectively. Furthermore, since the focus of the SOTA BPTT-type methods (\eg, surrogate function, network architecture, and regularization) are orthogonal to ours, our training techniques can be plugged into their methods to achieve better training efficiency. 
Some ANN-to-SNN-based and spike representation-based methods \cite{yang2022training,bu2022optimal,li2021free} also achieve satisfactory accuracy with relatively small training costs. However, they typically require a (much) larger number of time steps (\cref{table:sota}), which hurts the energy efficiency for neuromorphic computing.

\subsection{Influence of $T$ and $\tau$}

For efficient training, the SLTT method approximates the gradient calculated by BPTT by ignoring the temporal components in \cref{eqn:u-update-final,eqn:u-update-inter}. So when $T$ or $\tau$ is large, the approximation may not be accurate enough. In this subsection, we conduct experiments with different $\tau$ and $T$ on the neuromorphic datasets, DVS-Gesture and DVS-CIFAR10. We verify that the proposed method can still work well for large $T$ and commonly used $\tau$ \cite{xiao2022online,deng2022temporal,zheng2020going,guo2022recdis}, as shown in \cref{fig:diff_tau_step}. Regarding large time steps, SLTT obtains similar accuracy with BPTT even when $T=50$, and SLTT can outperform BPTT when $T<30$ on the two neuromorphic datasets.
For different $\tau$, our method can consistently perform better than BPTT, although there is a performance drop for SLTT when $\tau$ is large. 

\section{Conclusion}
\label{sec:conclusion}
In this work, we propose the Spatial Learning Through Time (SLTT) method that significantly reduces the time and memory complexity compared with the vanilla BPTT with SG method. We first show that the backpropagation of SNNs through the temporal domain contributes a little to the final calculated gradients. By ignoring unimportant temporal components in gradient calculation and introducing an online calculation scheme, our method reduces the scalar multiplication operations and achieves time-step-independent memory occupation.
Additionally, thanks to the instantaneous gradient calculation in our method, we propose a variant of SLTT, called SLTT-K, that allows backpropagation only at $K$ time steps. SLTT-K can further reduce the time complexity of SLTT significantly. Extensive experiments on large-scale static and neuromorphic datasets demonstrate superior training efficiency and high performance of the proposed method, and illustrate the method's effectiveness under different network settings and large-scale network structures.

\section*{Acknowledgment}
Z.~Lin was supported by National Key R\&D Program of China (2022ZD0160302), the NSF China (No. 62276004), and Qualcomm. The work of Z.-Q. Luo was supported in part by the National Key Research and Development Project under grant 2022YFA1003900, and in part by the Guangdong Provincial Key Laboratory of Big Data Computing. The work of Yisen Wang was supported by the National Natural Science Foundation of China (62006153) and the Open Research Projects of Zhejiang Lab (No. 2022RC0AB05).

{\small
\bibliographystyle{ieee_fullname}
\bibliography{egbib}
}

\appendix
\setcounter{table}{0}  
\setcounter{figure}{0}
\setcounter{equation}{0}
\renewcommand{\theequation}{S\arabic{equation}}
\renewcommand{\thetable}{S\arabic{table}}
\renewcommand{\thefigure}{S\arabic{figure}}
\renewcommand{\thealgorithm}{S\arabic{algorithm}}

\section{Detailed Dynamics of Feedforward SNNs with LIF Neurons}
An SNN operates by receiving an input time sequence and generating an output time sequence (binary spike train) through the iterative use of brain-inspired neuronal dynamics.
We consider feedforward SNNs with LIF neurons in this work. Specifically, the network follows the difference equation:
\begin{equation} \label{eqn:dynamics_appendix}
\left\{
\begin{aligned}
&\mathbf{u}^{l}[t]=(1-\frac{1}{\tau})\mathbf{v}^{l}[t-1] + \mathbf{W}^{l} \mathbf{s}^{l-1}[t], \\
&\mathbf{s}^{l}[t]=H(\mathbf{u}^{l}[t]-V_{th}),  \\
&\mathbf{v}^{l}[t]=\mathbf{u}^{l}[t] - V_{th}\mathbf{s}^{l}[t],
\end{aligned}
\right.
\end{equation}
where $l=1,\cdots,L$ is the layer index, $t=1,\cdots,T$ is the time step index, $\{\mathbf{s}^{0}[t]\}_{t=1}^T$ are the input to the network, and $\{\mathbf{s}^{l}[t]\}_{t=1}^T$ are the binary output sequence of the $l$-th layer, $\tau$ is a pre-defined time constant, $H(\cdot)$ is the the Heaviside step function, $\mathbf{u}^{l}[t]$ is the membrane potential before resetting, $\mathbf{v}^{l}[t]$ is the potential after resetting, and $\mathbf{W}^{l}$ are the weight to be trained. $\mathbf{v}^{l}[0]$ is set to be 0 for $l=1,\cdots,L$. 

The input $\mathbf{s}^0$ to the SNN can consist of either neuromorphic data or static data such as images. While neuromorphic data are inherently suitable for SNNs, when dealing with static data, a common approach is to repeatedly apply them to the first layer at each time step \cite{xiao2021ide,zheng2020going,rueckauer2017conversion,meng2022training}. This allows the first layer to serve as a spike-train data generator.

The output of the network $\{\mathbf{s}^{L}[t]\}_{t=1}^T$ is utilized for decision making. In classification tasks, a common approach involves computing $\mathbf{o} = \frac{1}{T}\sum_{t=1}^{T} \mathbf{W}^o\mathbf{s}^{L}[t]$, where $\mathbf{o} \in \mathbb{R}^c$ and $c$ represents the number of classes. The input sequence $\{\mathbf{s}^{0}[t]\}_{t=1}^T$ is classified as belonging to the $i$-th class if $\mathbf{o}_i$ is the largest value among $\{\mathbf{o}_1,\cdots,\mathbf{o}_c\}$.

Due to the binary nature of spike communication between neurons, SNNs can be efficiently implemented on neuromorphic chips, enabling energy-efficient applications.

\section{Derivation for \cref{eqn:u-update-final,eqn:u-update-inter}}
Recall that  
\begin{equation} \small \label{eqn:epsilon}
\textcolor{black}{\mathbf{\epsilon}^{l}[t]} \triangleq
\frac{\partial \mathbf{u}^{l}[t+1]}{\partial \mathbf{u}^{l}[t]}
+\frac{\partial \mathbf{u}^{l}[t+1]}{\partial \mathbf{s}^{l}[t]}
\frac{\partial \mathbf{s}^{l}[t]}{\partial \mathbf{u}^{l}[t]},
\end{equation}
which is the dependency between $\mathbf{u}^l[t+1]$ and $\mathbf{u}^l[t]$.
We derive \cref{eqn:u-update-inter} as below. We omit the derivation for \cref{eqn:u-update-final} since it is a simple corollary of \cref{eqn:u-update-inter}.

\begin{lemma}[\cref{eqn:u-update-inter}]
	\begin{equation} \small \label{eqn:u-update-inter-supp}
	\begin{aligned}
	\frac{\partial \mathcal{L}}{\partial \mathbf{u}^{l}[t]}
	=&\textcolor{newblue}{
		\frac{\partial \mathcal{L}}{\partial \mathbf{u}^{l+1}[t]}
		\frac{\partial \mathbf{u}^{l+1}[t]}{\partial \mathbf{s}^{l}[t]}
		\frac{\partial \mathbf{s}^{l}[t]}{\partial \mathbf{u}^{l}[t]}}
	\\
	&+ \textcolor{newgreen}{
		\sum_{t^\prime=t+1}^{T}
		\frac{\partial \mathcal{L}}{\partial \mathbf{u}^{l+1}[t^\prime]}
		\frac{\partial \mathbf{u}^{l+1}[t^\prime]}{\partial \mathbf{s}^{l}[t^\prime]}
		\frac{\partial \mathbf{s}^{l}[t^\prime]}{\partial \mathbf{u}^{l}[t^\prime]}
		\prod_{t^{\dprime}=1}^{t^\prime - t}
		\mathbf{\epsilon}^{l}[t^\prime-t^\dprime]
	}.
	\end{aligned}
	\end{equation}
\end{lemma}

\begin{proof}
	According to \cref{fig:bptt} and the chain rule, we have  
	\begin{equation} \small \label{eqn:chain-rule1}
	\textcolor{red}{\frac{\partial \mathcal{L}}{\partial \mathbf{u}^{l}[T]}}
	=\textcolor{black}{
		\frac{\partial \mathcal{L}}{\partial \mathbf{u}^{l+1}[T]}
		\frac{\partial \mathbf{u}^{l+1}[T]}{\partial \mathbf{s}^{l}[T]}
		\frac{\partial \mathbf{s}^{l}[T]}{\partial \mathbf{u}^{l}[T]}
	},
	\end{equation}
	and
	\begin{equation} \small \label{eqn:chain-rule2}
	\begin{aligned}
	\textcolor{red}{\frac{\partial \mathcal{L}}{\partial \mathbf{u}^{l}[t]}}
	=&\textcolor{black}{
		\frac{\partial \mathcal{L}}{\partial \mathbf{u}^{l+1}[t]}
		\frac{\partial \mathbf{u}^{l+1}[t]}{\partial \mathbf{s}^{l}[t]}
		\frac{\partial \mathbf{s}^{l}[t]}{\partial \mathbf{u}^{l}[t]}
	}
	+ \textcolor{red}{\frac{\partial \mathcal{L}}{\partial \mathbf{u}^{l}[t+1]}}
	\mathbf{\epsilon}^l[t]
	,
	\end{aligned}
	\end{equation}
	when $t = T-1, \cdots, 1$. \cref{eqn:chain-rule1,eqn:chain-rule2} are standard steps in the Backpropagation through time (BPTT) algorithm.
	
	Then we drive \cref{eqn:u-update-inter-supp} from \cref{eqn:chain-rule2} by induction w.r.t. $t$. When $t=T-1$,
	\begin{equation} \small 
	\begin{aligned}
	\textcolor{red}{\frac{\partial \mathcal{L}}{\partial \mathbf{u}^{l}[T-1]}}
	=&\textcolor{black}{
		\frac{\partial \mathcal{L}}{\partial \mathbf{u}^{l+1}[T-1]}
		\frac{\partial \mathbf{u}^{l+1}[T-1]}{\partial \mathbf{s}^{l}[T-1]}
		\frac{\partial \mathbf{s}^{l}[T-1]}{\partial \mathbf{u}^{l}[T-1]}
	}
	+ \textcolor{red}{\frac{\partial \mathcal{L}}{\partial \mathbf{u}^{l}[T]}}
	\mathbf{\epsilon}^l[T-1]
	,
	\end{aligned}
	\end{equation}
	which satisfies \cref{eqn:u-update-inter-supp}.
	When $t<T-1$, we assume \cref{eqn:u-update-inter-supp} is satisfied for $t+1$, then show that \cref{eqn:u-update-inter-supp} is satisfied for $t$: 
	\begin{equation} \small \label{eqn:induction}
	\hspace{-1.7em}
	\begin{aligned}
	\textcolor{red}{\frac{\partial \mathcal{L}}{\partial \mathbf{u}^{l}[t]}}
	=&\textcolor{newblue}{
		\frac{\partial \mathcal{L}}{\partial \mathbf{u}^{l+1}[t]}
		\frac{\partial \mathbf{u}^{l+1}[t]}{\partial \mathbf{s}^{l}[t]}
		\frac{\partial \mathbf{s}^{l}[t]}{\partial \mathbf{u}^{l}[t]}
	}
	+ \textcolor{red}{\frac{\partial \mathcal{L}}{\partial \mathbf{u}^{l}[t+1]}}
	\mathbf{\epsilon}^l[t]
	\\
	=&\textcolor{newblue}{
		\frac{\partial \mathcal{L}}{\partial \mathbf{u}^{l+1}[t]}
		\frac{\partial \mathbf{u}^{l+1}[t]}{\partial \mathbf{s}^{l}[t]}
		\frac{\partial \mathbf{s}^{l}[t]}{\partial \mathbf{u}^{l}[t]}
	}\\
	&+ \textcolor{red}{
		\left(
		\frac{\partial \mathcal{L}}{\partial \mathbf{u}^{l+1}[t+1]}
		\frac{\partial \mathbf{u}^{l+1}[t+1]}{\partial \mathbf{s}^{l}[t+1]}
		\frac{\partial \mathbf{s}^{l}[t+1]}{\partial \mathbf{u}^{l}[t+1]}\right.
	}
	\\
	& \quad \quad \textcolor{red}{+} 
	\textcolor{red}{
		\left.\sum_{t^\prime=t+2}^{T}
		\frac{\partial \mathcal{L}}{\partial \mathbf{u}^{l+1}[t^\prime]}
		\frac{\partial \mathbf{u}^{l+1}[t^\prime]}{\partial \mathbf{s}^{l}[t^\prime]}
		\frac{\partial \mathbf{s}^{l}[t^\prime]}{\partial \mathbf{u}^{l}[t^\prime]}
		\prod_{t^{\dprime}=1}^{t^\prime - t - 1}
		\mathbf{\epsilon}^{l}[t^\prime-t^\dprime]
		\right)} \mathbf{\epsilon}^l[t]
	\\
	=&\textcolor{newblue}{ 
		\frac{\partial \mathcal{L}}{\partial \mathbf{u}^{l+1}[t]}
		\frac{\partial \mathbf{u}^{l+1}[t]}{\partial \mathbf{s}^{l}[t]}
		\frac{\partial \mathbf{s}^{l}[t]}{\partial \mathbf{u}^{l}[t]}
	} + 
	\frac{\partial \mathcal{L}}{\partial \mathbf{u}^{l+1}[t+1]}
	\frac{\partial \mathbf{u}^{l+1}[t+1]}{\partial \mathbf{s}^{l}[t+1]}
	\frac{\partial \mathbf{s}^{l}[t+1]}{\partial \mathbf{u}^{l}[t+1]} 
	\mathbf{\epsilon}^l[t]
	\\
	& +
	\textcolor{black}{
		\sum_{t^\prime=t+2}^{T}
		\frac{\partial \mathcal{L}}{\partial \mathbf{u}^{l+1}[t^\prime]}
		\frac{\partial \mathbf{u}^{l+1}[t^\prime]}{\partial \mathbf{s}^{l}[t^\prime]}
		\frac{\partial \mathbf{s}^{l}[t^\prime]}{\partial \mathbf{u}^{l}[t^\prime]}
		\left(
		\prod_{t^{\dprime}=1}^{t^\prime - t - 1}
		\mathbf{\epsilon}^{l}[t^\prime-t^\dprime]
		\right)
		\mathbf{\epsilon}^l[t]
	}
	\\
	=&\textcolor{newblue}{
		\frac{\partial \mathcal{L}}{\partial \mathbf{u}^{l+1}[t]}
		\frac{\partial \mathbf{u}^{l+1}[t]}{\partial \mathbf{s}^{l}[t]}
		\frac{\partial \mathbf{s}^{l}[t]}{\partial \mathbf{u}^{l}[t]}}
	\\
	&+ \textcolor{newgreen}{
		\sum_{t^\prime=t+1}^{T}
		\frac{\partial \mathcal{L}}{\partial \mathbf{u}^{l+1}[t^\prime]}
		\frac{\partial \mathbf{u}^{l+1}[t^\prime]}{\partial \mathbf{s}^{l}[t^\prime]}
		\frac{\partial \mathbf{s}^{l}[t^\prime]}{\partial \mathbf{u}^{l}[t^\prime]}
		\prod_{t^{\dprime}=1}^{t^\prime - t}
		\mathbf{\epsilon}^{l}[t^\prime-t^\dprime]
	}
	,
	\end{aligned}
	\end{equation}
	where the first equation is due to \cref{eqn:chain-rule2}, and the second equation is due to the assumption that \cref{eqn:u-update-inter-supp} is satisfied for $t+1$.
\end{proof}

\begin{algorithm}[h] 
	\caption{One iteration of SNN training with the BPTT with SG method.}
	\label{alg:bptt}
	\begin{algorithmic}[1] 
		\Require Time steps $T$; Network depth $L$; Network parameters $\{\mathbf{W}^l\}_{l=1}^L$; Training data $(\mathbf{s}^0,\mathbf{y})$; Learning rate $\eta$.
		\State //\textbf{Forward:}
		\For {$t=1,2,\cdots,T$}
		\For {$l=1,2,\cdots,L$}
		\State Calculate $\mathbf{u}^l[t]$ and $\mathbf{s}^l[t]$ by \cref{eqn:appendix_forward,eqn:appendix_spike};
		\EndFor
		\EndFor
		
		\State Calculate the loss $\mathcal{L}$ based on $\mathbf{s}^L$ and $\mathbf{y}$.
		\State //\textbf{Backward:}
		\State $\mathbf{e}_\mathbf{s}^L[T] = \frac{\partial \mathcal{L}}{\partial \mathbf{s}^{L}[T]}$;
		\For {$l=L-1,\cdots,1$}
		\State $\mathbf{e}_\mathbf{u}^l[T] = \mathbf{e}_\mathbf{s}^{l+1}[T] \frac{\partial \mathbf{s}^{l+1}[T]}{\partial \mathbf{u}^{l}[T]}$, 
		$\mathbf{e}_\mathbf{s}^l[T] = \mathbf{e}_\mathbf{u}^l[T] \frac{\partial \mathbf{u}^{l}[T]}{\partial \mathbf{s}^{l}[T]}$;
		\EndFor	
		
		\For {$t=T-1 ,T-2,\cdots,1$}
		\For {$l=L ,L-1,\cdots,1$}
		\If {$l=L$}
		\State $\mathbf{e}_\mathbf{s}^L[t] = 
		\frac{\partial \mathcal{L}}{\partial \mathbf{s}^{L}[t]} + \mathbf{e}_\mathbf{u}^L[t+1]
		\frac{\partial \mathbf{u}^{L}[t+1]}{\partial \mathbf{s}^{L}[t]}
		$;
		\Else
		\State $\mathbf{e}_\mathbf{s}^l[t] = \mathbf{e}_\mathbf{u}^{l+1}[t]
		\frac{\partial \mathbf{u}^{l+1}[t]}{\partial \mathbf{s}^{l}[t]} + \mathbf{e}_\mathbf{u}^l[t+1]
		\frac{\partial \mathbf{u}^{l}[t+1]}{\partial \mathbf{s}^{l}[t]}
		$;
		\EndIf
		\State 
		$\mathbf{e}_\mathbf{u}^l[t] = \mathbf{e}_\mathbf{s}^{l}[t]
		\frac{\partial \mathbf{s}^{l}[t]}{\partial \mathbf{u}^{l}[t]} + \mathbf{e}_\mathbf{u}^l[t+1]
		\frac{\partial \mathbf{u}^{l}[t+1]}{\partial \mathbf{u}^{l}[t]}
		$;
		\State $\Delta \mathbf{W}^l \mathrel{+}= \mathbf{e}_\mathbf{u}^{l}[t] ^\top \mathbf{s}^{l-1}[t]^\top $;
		\EndFor	
		\EndFor	
		
		\State $\mathbf{W}^l = \mathbf{W}^l - \eta \Delta \mathbf{W}^l, \ l=1,2,\cdots,L$;
		\Ensure Trained network parameters $\{\mathbf{W}^l\}_{l=1}^L$.
	\end{algorithmic}
\end{algorithm} 

\section{Time Complexity Analysis of SLTT and BPTT}

\subsection{Pseudocode of the Bachpropagation Though Time with Surrogate Gradinet Method}
We present the pseudocode of one iteration of SNN training with the bachpropagation though time (BPTT) with surrogate gradinet (SG) method in \cref{alg:bptt}. Note that the forward pass is defined by
\begin{equation} \label{eqn:appendix_forward}
\mathbf{u}^{l}[t]=(1-\frac{1}{\tau})(\mathbf{u}^{l}[t-1] -V_{t h} \mathbf{s}^{l}[t-1])+ \mathbf{W}^{l} \mathbf{s}^{l-1}[t],
\end{equation}
where $\mathbf{s}^l$ are the output spike trains of the $l^{\text{th}}$ layer, which are calculated by:
\begin{equation} \label{eqn:appendix_spike}
\mathbf{s}^{l}[t]=H(\mathbf{u}^{l}[t] -V_{t h}).
\end{equation}

\begin{algorithm}[h] 
	\caption{One iteration of SNN training with the SLTT or SLTT-K methods.}
	\label{alg:appendix_SLTT}
	\begin{algorithmic}[1] 
		\Require Time steps $T$; Network depth $L$; Network parameters $\{\mathbf{W}^l\}_{l=1}^L$; Training data $(\mathbf{s}^0,\mathbf{y})$; Learning rate $\eta$; Required backpropagation times $K$ (for SLTT-K).
		\item[\textbf{Initialize:}] $\Delta \mathbf{W}^l = 0, \ l=1,2,\cdots,L$.
		\If {using SLTT-K}
		\State Sample $K$ numbers in $[1,2,\cdots,T]$ w/o replacement to form $required\_bp\_steps$;
		\Else
		\State $required\_bp\_steps=[1,2,\cdots,T]$;
		\EndIf
		
		\For {$t=1,2,\cdots,T$}
		\State Calculate $\mathbf{s}^L[t]$ by \cref{eqn:appendix_forward,eqn:appendix_spike}; \quad //\textbf{Forward}
		\State Calculate the instantaneous loss $\ell$;
		\If {$t$ in $required\_bp\_steps$} \quad\quad //\textbf{Backward}
		\State $\mathbf{e}_\mathbf{u}^{L}[t] =\frac{1}{T}\frac{\partial \ell}{\partial \mathbf{s}^{L}[t]}\frac{\partial \mathbf{s}^{L}[t]}{\partial \mathbf{u}^{L}[t]}$;
		\For {$l=L-1,\cdots,1$}
		\State $\mathbf{e}_\mathbf{u}^l[t] = \mathbf{e}_\mathbf{u}^{l+1}[t]
		\frac{\partial \mathbf{u}^{l+1}[t]}{\partial \mathbf{s}^{l}[t]}
		\frac{\partial \mathbf{s}^{l}[t]}{\partial \mathbf{u}^{l}[t]}$;
		\State $\Delta \mathbf{W}^l \mathrel{+}= \mathbf{e}_\mathbf{u}^{l}[t] ^\top \mathbf{s}^{l-1}[t]^\top$;
		\EndFor
		\EndIf
		\EndFor
		
		\State $\mathbf{W}^l = \mathbf{W}^l - \eta \Delta \mathbf{W}^l, \ l=1,2,\cdots,L$;
		
		\Ensure Trained network parameters $\{\mathbf{W}^l\}_{l=1}^L$.
	\end{algorithmic}
\end{algorithm} 

\subsection{Time Complexity Analysis}
The time complexity of each time step is dominated by the number of scalar multiplication operations. In this subsection, we analyze the required scalar multiplications of the Spatial Learning
Through Time (SLTT) and BPTT with SG methods. We show the pseudocode of SLTT in \cref{alg:appendix_SLTT} again for better presentation.

Consider that each layer has $d$ neurons. For simplicity, we only consider the scalar multiplications for one intermediate time step ($t<T$) and one intermediate layer ($l<L$). 
Regarding the BPTT with SG method, it requires scalar multiplications to update $\mathbf{e}_\mathbf{s}^l[t]$, $\mathbf{e}_\mathbf{u}^l[t]$, and  $\Delta \mathbf{W}^l$ at lines 18, 20, and 21 respectively in \cref{alg:bptt}. To update $\mathbf{e}_\mathbf{s}^l[t]$, two vector-Jacobian products are required. Since $\frac{\partial \mathbf{u}^{l}[t+1]}{\partial \mathbf{s}^{l}[t]}$ is a diagonal matrix, the number of scalar multiplications for updating  $\mathbf{e}_\mathbf{s}^l[t]$ is $d^2+d$. To update $\mathbf{e}_\mathbf{u}^l[t]$ at line 20, the number is $2d$, since the Jacobians $\frac{\partial \mathbf{s}^{l}[t]}{\partial \mathbf{u}^{l}[t]}$ and $\frac{\partial \mathbf{u}^{l}[t+1]}{\partial \mathbf{u}^{l}[t]}$ are both diagonal. It requires $d^2$ scalar multiplications to update $\Delta \mathbf{W}^l$.
Regarding the SLTT method, it requires scalar multiplications to update $\mathbf{e}_\mathbf{u}^l[t]$ and  $\Delta \mathbf{W}^l$ at lines 12 and 13, respectively, in \cref{alg:appendix_SLTT}. It requires $d^2+d$ scalar multiplications to update $\mathbf{e}_\mathbf{u}^l[t]$, and requires $d^2$ to update $\Delta \mathbf{W}^l$. Then compared with the BPTT with SG method, SLTT reduces the number of scalar multiplications by $2d$ for one intermediate time step and one intermediate layer. For SLTT-K, it requires updating $\mathbf{e}_\mathbf{u}^l[t]$ and  $\Delta \mathbf{W}^l$ only at $K$ randomly chosen time steps. Then the required number of scalar multiplication operations is proportional to $K$, which is proportional to $T$ for SLTT and BPTT.

\section{Hard Reset Mechanism}

In this work, we adopt the LIF model with soft reset as the neuron model, as shown in \cref{eqn:lif}. Besides the soft reset mechanism, hard reset is also applicable for our method. Specifically, for the LIF model with hard reset
\begin{equation} \small \label{eqn:lif-hard}
\left\{ 
\begin{aligned}
&u[t]=(1-\frac{1}{\tau})v[t-1] + \sum_{i}w_i s_i[t] + b, \\[-3pt]
&s_{out}[t]=H(u[t]-V_{th}),  \\
&v[t]=u[t] \cdot (1-s_{out}[t]),
\end{aligned}
\right. 
\end{equation}
we can also observe that the temporal components contribute a little to $\frac{\partial \mathcal{L}}{\partial \mathbf{u}^{l}[t]}$ and the observation in \cref{sec:observation} still holds. Consider the rectangle surrogate (\cref{eqn:rectangle_sg}) with $\gamma=V_{th}$, the diagonal elements of $\textcolor{black}{\mathbf{\epsilon}^{l}[t]}$, which is the dependency between $\mathbf{u}^l[t+1]$ and $\mathbf{u}^l[t]$, become
\begin{equation} \hspace{-15pt} \small \label{eqn:small_sensitivity_hard}
\textcolor{black}{\left(\mathbf{\epsilon}^{l}[t]\right)_{jj}} = \left\{\begin{array}{l} \lambda\left(1-(\mathbf{s}^l[t])_j-\frac{(\mathbf{u}^l[t])_j}{V_{th}}\right), \quad \frac{1}{2}V_{th}<\left(\mathbf{u}^{l}[t]\right)_j<\frac{3}{2}V_{th}, \\ \lambda(1-(\mathbf{s}^l[t])_j), \quad \quad \ \ \quad \quad \quad \quad \  \text{otherwise}.\end{array}\right.
\end{equation}
Since $(\mathbf{s}^l[t])_j \in \{0,1\}$, we have $\textcolor{black}{\left(\mathbf{\epsilon}^{l}[t]\right)_{jj}} \in (-\frac{3}{2}\lambda, \frac{1}{2}\lambda)\cup\{\lambda\}$, which is at least not large for commonly used small $\lambda$. As a result, the spatial components in \cref{eqn:u-update-final,eqn:u-update-inter} dominate the gradients and then we can ignore the temporal components without much performance drop.

\begin{table}[t] 
	\caption{Comparison of accuracy between soft reset and hard reset on CIFAR-10 and DVS-CIFAR10.}
	\label{table:hard-soft}
	\centering
	\begin{threeparttable}
		\begin{tabular}{ccc}
			\toprule  
			Dataset  & Reset Mechanism & Acc  \\
			\midrule 
			\multirow{2}*{CIFAR-10}&  Soft & $94.44\%\pm0.21\%$  \\
			&  Hard & $94.34\%$   \\ 
			\hline
			\multirow{2}*{DVS-CIFAR10} &  Soft & $82.20\pm0.95\%$  \\
			&  Hard & $81.40\%$  \\
			\bottomrule
		\end{tabular}
	\end{threeparttable}
\end{table}

We conduct experiments with the hard reset mechanism on CIFAR-10 and DVS-CIFAR10, using the same training settings as for soft reset. And the comparison between hard reset and soft reset is shown in \cref{table:hard-soft}. We can see that our method can also achieve competitive results with hard reset.

\section{Performance Under the Smaller $T$ Setting}
We compare SLTT and BPTT with SG for smaller $T$ on CIFAR-10. The results are shown in \cref{table:small_t}. As $T$ decreases, the behaviors of both methods tend to converge, with SLTT consistently demonstrates better time and memory efficiency than BPTT with SG.

\begin{table}[t] 
	\caption{Comparison between SLTT and BPTT on CIFAR-10.}
	\label{table:small_t}
	\centering
	\begin{threeparttable}
		\begin{tabular}{ccccc}
			\toprule  
			Time Steps & Method  & Memory & Time & Acc \\
			\midrule 
			\multirow{2}*{$T=2$}  &  BPTT & 1.47G & 1.81h
			& 93.90\% \\
			&  SLTT & \bf{1.10G} & \bf{1.80h} & \bf{93.96\%} \\
			\hline
			\multirow{2}*{$T=4$} & BPTT & 2.19G & 3.41s & \bf{94.29\%} \\
			&  SLTT & \bf{1.09G} &  \bf{3.29h} & 94.17\% \\
			\hline
			\multirow{2}*{$T=6$} &  BPTT & 3.00G & 6.35h & \bf{94.60\%} \\
			&  SLTT & \bf{1.09G} & \bf{4.58h} & 94.59\% \\
			\bottomrule
		\end{tabular}
	\end{threeparttable}
\end{table}

\begin{table*}[h] \footnotesize
	\caption{\small Comparisons with other SNN training methods using the same network architectures on CIFAR-10, CIFAR-100, ImageNet, DVS-Gesture, and DVS-CIFAR10.}
	\label{table:sota-group}
	\centering
	\vspace{-7pt}
	\begin{threeparttable}
		\begin{tabular}{c|lcccc}
			\toprule[1.08pt] & Method & Network & Time Steps & Efficient Training  & Mean$\pm$Std (Best) \\
			\midrule[1.08pt]
			
			\multirow{10}*{\rotatebox{90}{CIFAR-10}} 
			&Dspike\cite{li2021differentiable} & \multirow{5}*{ResNet-18}  & 6 & \XSolidBrush & $94.25\pm0.07\%$  \\
			& TET\cite{deng2022temporal} &  & 6  & \XSolidBrush  & ${94.50\pm0.07\%}$  \\
			& DSR \cite{meng2022training} &  & 20 & \Checkmark & $95.40\pm0.15\%$ \\
			&RecDis\cite{guo2022recdis} &  & 6 & \XSolidBrush  & $\mathbf{95.55\pm0.05 \%}$ \\
			&\textbf{SLTT (ours)}  &  & 6 & \Checkmark & ${94.44\%\pm0.21\%}$ (${94.59\%}$) \\
			\cline{2-6}
			&LTL-Online \cite{yang2022training} \tnote{1} & \multirow{4}*{VGG-16} & 16 & \Checkmark  & ${92.85\%}$ \\
			&DIET-SNN \cite{rathi2020diet} \tnote{1} &  & 10 & \XSolidBrush  & ${93.44\%}$ \\
			&Temporal Pruning \cite{chowdhurytowards} &  & 5 & \XSolidBrush  & $\mathbf{93.90\%}$ \\
			&\textbf{SLTT (ours)}  &  & 6 & \Checkmark & ${93.28\%\pm0.02\%}$ (${93.29\%}$) \\
			\midrule[1.08pt]
			
			\multirow{9}*{\rotatebox{90}{CIFAR-100}} 
			&RecDis\cite{guo2022recdis} & \multirow{5}*{ResNet-18} & 4 & \XSolidBrush  & $74.10\pm0.13 \%$ \\
			&TET\cite{deng2022temporal} & & 6 & \XSolidBrush  & ${74.72\pm0.28\%}$  \\
			& DSR \cite{meng2022training} & & 20 & \Checkmark & $\mathbf{78.50\pm0.12\%}$ \\
			&Dspike\cite{li2021differentiable} &  & 6 & \XSolidBrush & $74.24\pm0.10\%$  \\
			&\textbf{SLTT (ours)} & & 6 & \Checkmark  & ${74.38\%\pm0.30\% \ (74.67\%)}$  \\
			\cline{2-6}
			&ANN-to-SNN\cite{bu2022optimal} \tnote{1} & \multirow{4}*{VGG-16} & 8 & \Checkmark  & $\mathbf{73.96}\%$ \\
			&DIET-SNN \cite{rathi2020diet} \tnote{1} & & 10 & \XSolidBrush  & ${69.67\%}$ \\
			&Temporal Pruning \cite{chowdhurytowards} &  & 5 & \XSolidBrush  & ${71.58\%}$ \\
			&\textbf{SLTT (ours)}  &  & 6 & \Checkmark  & $\underline{72.55\%\pm0.24\% \ (72.83\%)}$  \\
			\midrule[1.08pt]
			
			\multirow{13}*{\rotatebox{90}{ImageNet}} 
			& ANN-to-SNN \cite{meng2022ann} \tnote{1} & \multirow{6}*{ResNet-34}  & 256, 512 & \Checkmark  & $73.16\%, {74.18\%}$ \\	
			&ANN-to-SNN\cite{li2021free} \tnote{1} &  & 32, 256 & \Checkmark & ${64.54\%, \textbf{74.61\%}}$ \\
			&TET\cite{deng2022temporal} & & 6 & \XSolidBrush  & $64.79\%$ \\
			&OTTT\cite{xiao2022online} & & 6 & \Checkmark  & ${65.15\%}$ \\
			&SEW \cite{fang2021sew} &  & 4 & \XSolidBrush  & $67.04\%$ \\
			&\textbf{SLTT (ours)}  & & 6 & \Checkmark  & ${66.19\%}$  \\
			\cline{2-6}
			&STBP-tdBN \cite{zheng2020going} & \multirow{4}*{ResNet-50}  & 6 & \XSolidBrush  & $64.88\%$ \\
			&SEW \cite{fang2021sew} &  & 4 & \XSolidBrush  & ${67.78\%}$ \\
			& ANN-to-SNN \cite{meng2022ann} \tnote{1} &  & 256, 512 & \Checkmark  & ${73.56\%},\mathbf{75.04\%}$ \\
			&\textbf{SLTT (ours)}  & & 6 & \Checkmark  & ${67.02\%}$  \\
			\cline{2-6}
			& ANN-to-SNN \cite{meng2022ann} \tnote{1} & \multirow{4}*{ResNet-101}  & 256, 512 & \Checkmark  & $\mathbf{73.50\%,75.72\%}$ \\
			&SEW \cite{fang2021sew} &  & 4 & \XSolidBrush  & $68.76\%$ \\
			&\textbf{SLTT-2 (ours)}  &  & 6 & \Checkmark  & ${69.26\%}$  \\
			\midrule[1.08pt]
			
			\multirow{4}*{\rotatebox{90}{\small DVS-Gesture}} 
			&STBP-tdBN \cite{zheng2020going} & \multirow{2}*{ResNet-18} & 40 & \XSolidBrush  &  $96.87\%$ \\
			& SLTT (ours) & & 20  & \Checkmark  & $\mathbf{97.68\pm0.53\% \ (98.26\%)}$  \\
			\cline{2-6}
			&OTTT \cite{xiao2022online} & \multirow{2}*{VGG-11} & 20 & \Checkmark  &  $96.88\%$  \\
			&\textbf{SLTT (ours)}  &  & 20 & \Checkmark  & $\mathbf{98.50\pm0.21\% \ (98.62\%)}$ \\
			\midrule[1.08pt]
			
			\multirow{8}*{\rotatebox{90}{DVS-CIFAR10}} 
			&STBP-tdBN\cite{zheng2020going} & \multirow{4}*{ResNet-18} & 10 & \XSolidBrush  & $67.80\%$  \\
			&Dspike\cite{li2021differentiable} &  & 10 & \XSolidBrush  & $75.40\pm0.05\%$  \\
			& InfLoR\cite{guo2022reducing} &  & 10 & \XSolidBrush   & $75.50\pm0.12\%$  \\
			&\textbf{SLTT (ours)}  &  & 10 & \Checkmark & $\mathbf{81.87\pm0.49\% \ (82.20\%)}$ \\
			\cline{2-6}	
			& DSR \cite{meng2022training} & \multirow{4}*{VGG-11} & 20 & \Checkmark & $77.27\pm0.24\%$ \\
			&OTTT \cite{xiao2022online} &  & 10 & \Checkmark  &  $76.27\pm0.05\% (76.30\%)$  \\
			& TET\cite{deng2022temporal} &   & 10 & \XSolidBrush   & $\mathbf{83.17\pm0.15\%}$  \\
			&\textbf{SLTT (ours)}  &    & 10 & \Checkmark & ${82.20\pm0.95\% \ (83.10\%)}$ \\
			\bottomrule[1.08pt]
			
		\end{tabular}
		\small
		$^{1}$ Pre-trained ANN models are required.
	\end{threeparttable}
\end{table*}

\section{Comparison with the SOTA Using the Same Network Architectures}

We compare our method with other SOTA using the same network architectures, and the results are shown in \cref{table:sota-group}.
Our method achieves competitive results compared with the SOTA methods, while enabling efficient training at the same time. Furthermore, our method works on different network backbones, showing its effectiveness and applicability.

\section{Dataset Description and Preprocessing}

\paragraph{CIFAR-10} The CIFAR-10 dataset~\cite{krizhevsky2009learning} contains 60,000 32$\times$32 color images in 10 different classes, with 50,000 training samples and 10,000 testing samples. We normalize the image data to ensure that input images have zero mean and unit variance. We apply random cropping with 4 padding on each border of the image, random horizontal flipping, and cutout \cite{devries2017improved} for data augmentation. We apply direct encoding \cite{rathi2020diet} to encode the image pixels into time series. Specifically, the pixel values are repeatedly applied to the input layer at each time step. CIFAR-10 is licensed under MIT.

\paragraph{CIFAR-100} The CIFAR-100 dataset ~\cite{krizhevsky2009learning} is similar to CIFAR-10 except that it contains 100 classes of objects. There are 50,000 training samples and 10,000 testing samples, each of which is a 32$\times$32 color image. CIFAR-100 is licensed under MIT. We adopt the same data preprocessing and input encoding as CIFAR-10. 

\paragraph{ImageNet} The ImageNet-1K dataset~\cite{deng2009imagenet} contains color images in 1000 classes of objects, with 1,281,167 training images and 50,000 validation images. This dataset is licensed under Custom (non-commercial).  We normalize the image data to ensure that input images have zero mean and unit variance. For training samples, we apply random resized cropping to get images with size 224$\times$224, and then apply horizontal flipping. For validation samples, we resize the images to 256$\times$256 and then center-cropped them to 224$\times$224. We transform the pixels into time sequences by direct encoding \cite{rathi2020diet}, as done for CIFAR-10 and CIFAR-100.

\paragraph{DVS-Gesture} The DVS-Gesture \cite{amir2017low} dataset is recorded using a Dynamic Vision Sensor (DVS), consisting of spike trains with two channels corresponding to ON- and OFF-event spikes. The dataset contains 11 hand gestures from 29 subjects under 3 illumination conditions, with 1176 training samples and 288 testing samples. The license of  DVS-Gesture is Creative Commons Attribution 4.0. For data preprocessing, we follow \cite{fang2021incorporating} to integrate the events into frames. The event-to-frame integration is handled with the SpikingJelly \cite{SpikingJelly} framework.

\paragraph{DVS-CIFAR10} The DVS-CIFAR10 dataset~\cite{li2017cifar10} is a neuromorphic dataset converted from CIFAR-10 using a DVS camera.
It contains 10,000 event-based images with pixel dimensions expanded to 128$\times$128. The dataset is licensed under CC BY 4.0. We split the whole dataset into 9000 training images and 1000 testing images. Regarding data preprocessing, we integrate the events into frames \cite{fang2021incorporating}, and reduce the spatial resolution into 48$\times$48 by interpolation. For some experiments, we take random horizontal flip and random roll within 5 pixels as data augmentation, the same as \cite{deng2022temporal}.

\section{Network Architectures}
\label{sec:network}

\subsection{Scaled Weight Standardization}

An important characteristic of the proposed SLTT method is the instantaneous gradient calculation at each time step, which enables time-steps-independent memory costs. Under our instantaneous update framework, an effective technique, batch normalization (BN) along the temporal dimension \cite{zheng2020going,li2021differentiable,deng2022temporal,meng2022training}, cannot be adopted to our method, since this technique requires gathering data from all time steps to calculate the mean and variance statistics. For some tasks, we still use BN components, but their calculated statistics are based on data from each time step. For other tasks, we replace BN with scaled weight standardization (sWS) \cite{qiao2019micro,brock2021characterizing,brock2021high}, as introduced below.

The sWS component \cite{brock2021characterizing}, which is modified from the original weight standardization \cite{qiao2019micro}, normalizes the weights according to:
\begin{equation}
\hat{\mathbf{W}}_{i, j} = \gamma \frac{\mathbf{W}_{i, j}-\mu_{\mathbf{W}_{i, \cdot}}}{\sigma_{\mathbf{W}_{i, \cdot}}},
\end{equation}
where the mean $\mu_{\mathbf{W}_{i, \cdot}}$ and standard deviation $\sigma_{\mathbf{W}_{i, \cdot}}$ are calculated across the fan-in extent indexed by $i$, $N$ is the dimension of the fan-in extent, and $\gamma$ is a fixed hyperparameter. The hyperparameter $\gamma$ is set to stabilize the signal propagation in the forward pass. Specifically, for one network layer $\mathbf{z}=\hat{\mathbf{W}}g(\mathbf{x})$, where $g(\cdot)$ is the activation and the elements of $\mathbf{x}$ are i.i.d. from $\mathcal{N}(0,1)$, we determine $\gamma$ to make $\mathbb{E}(\mathbf{z})=0$, and $\operatorname{Cov}(\mathbf{z})=\mathbf{I}$. For SNNs, the activation $g(\cdot)$ at each time step can be treated as the Heaviside step function. Then we take $\gamma\approx2.74$ to stable the forward propagation, as calculated by \cite{xiao2022online}. Furthermore, 
sWS incorporate another learnable scaling factor for the weights \cite{brock2021characterizing,xiao2022online} to mimic the scaling factor of BN.
sWS shares similar effects with BN, but introduces no dependence of data from different batches and time steps. Therefore, sWS is a convincing alternative for replacing BN in our framework. 

\subsection{Normalization-Free ResNets}
\label{sec:nfnet}
For deep ResNets \cite{he2016deep}, sWS cannot enjoy the similar signal-preserving property as BN very well due to the skip connection. Then in some experiments, we consider the normalization-free ResNets (NF-ResNets) \cite{brock2021characterizing,brock2021high} that not only replace the BN components by sWS but also introduce other techniques to preserve the signal in the forward pass.

The NF-ResNets use the residual blocks of the form $x_{l+1}=x_{l}+\alpha f_{l}\left(x_{l} / \beta_{l}\right)$, where $\alpha$ and $\beta_l$ are hyperparameters used to stabilize signals, and the weights in $f_l(\cdot)$ are imposed with sWS.
The carefully determined $\beta_l$ and the sWS component together ensure that $f_l(x_l/\beta_l)$ has unit variance. $\alpha$ controls the rate of variance growth between blocks, and is set to be 0.2 in our experiments. Please refer to \cite{brock2021characterizing} for more details on the network design.

\subsection{Description of Adopted Network Architectures}

We adopt ResNet-18 \cite{he2016deep} with the pre-activation residual blocks \cite{he2016identity} to conduct experiments on CIFAR-10 and CIFAR-100. The channel sizes for the four residual blocks are 64, 128, 256, and 512, respectively. All the ReLU activations are substituted by the leaky integrate and fire (LIF) neurons. To make the network implementable for neuromorphic computing, we replace all the max pooling operations with average pooling. To enable instantaneous gradient calculation, we adopt the BN components that calculate the mean and variance statistics for each time step, not the total time horizon. Then for each iteration, a BN component is implemented for $T$ times, where $T$ is the number of total time steps.

We adopt VGG-11 \cite{simonyan2014very} to conduct experiments on DVS-Gesture and DVS-CIFAR10. As done for Resnet-18, we substitute all the max poolings with average poolings and use the time-step-wise BN. We remove two fully connected layers \cite{meng2022training,deng2022temporal,xiao2022online} to reduce the computation. A dropout layer \cite{srivastava2014dropout} is added behind each LIF neuron layer to bring better generalization. The dropout rates for DVS-Gesture and DVS-CIFAR10 are set to be 0.4 and 0.3, respectively.

\begin{table}[H]
	\caption{Training hyperparameters about optimization. ``WD'' means weight decay, ``LR'' means initial learning rate, and ``BS'' means batch size.}
	\label{table:parameter}
	\hspace{-0.7em}
	\begin{tabular}{l|cccc}
		\hline Dataset & Epoch & LR  & BS & WD\\
		\hline CIFAR-10   & 200 & $0.1$  & 128 & $5\times10^{-5}$  \\
		CIFAR-100 & 200 & $0.1$  & 128  & $5\times10^{-4}$ \\
		ImageNet (pre-train)  & 100 & $0.1$ & 256 & $1\times10^{-5}$ \\
		ImageNet (fine-tune) & 30 & $0.001$ & 256  & $0$ \\
		DVS-Gesture  & 300 & $0.1$ & 16 & $5\times10^{-4}$ \\
		DVS-CIFAR10  & 300 & $0.05$ & 128 & $5\times10^{-4}$ \\
		\hline
	\end{tabular}
\end{table}

We also adopt VGG-11 (WS) to conduct experiments on DVS-Gesture and achieve state-of-the-art performance. VGG-11 (WS) is a BN-free network that shares a similar architecture with VGG-11 introduced above. The difference between the two networks is that VGG-11 consists of the convolution-BN-LIF blocks while VGG-11 (WS) consists of the convolution-sWS-LIF blocks.

We adopt NF-ResNet-34, NF-ResNet-50, and NF-ResNet-101 to conduct experiments on ImageNet. Those networks are normalization-free ResNets introduced in \cref{sec:nfnet}. In Figs.~1 and 3 of the main content, the used network for ImageNet is NF-ResNet-34.

\section{Training Settings}

All the implementation is based on the PyTorch \cite{paszke2019pytorch} and SpikingJelly \cite{SpikingJelly} frameworks, and the experiments are carried out on one Tesla-V100 GPU or one Tesla-A100 GPU.

For CIFAR-10, CIFAR-100, and DVS-Gesture, we adopt the loss function proposed in \cite{deng2022temporal}:
\begin{equation} \label{eqn:appendix_loss}
\mathcal{L} = \frac{1}{T}\sum_{t=1}^{T} (1-\lambda) \ell_1(\mathbf{o}[t],y) + \lambda \ell_2(\mathbf{o}[t],V_{th}),
\end{equation}
where $T$ is the number of total time steps, $\ell_1$ is the cross entropy function, $\ell_2$ is the mean squared error (MSE) function, $\mathbf{o}[t]$ is the network output at the $t$-th time step, $y$ is the label, $\lambda$ is a hyperparameter taken as $0.05$, and $V_{th}$ is the spike threshold which is set to be $1$ in this work. For ImageNet, we use the same loss but simply set $\lambda=0$.
For DVS-CIFAR10, we also combine the cross entropy loss and the MSE loss, but the MSE loss does not act as a regularization term as in \cite{deng2022temporal}:
\begin{equation} 
\mathcal{L} = \frac{1}{T}\sum_{t=1}^{T} (1-\lambda) \ell_1(\mathbf{o}[t],y) + \lambda \ell_2(\mathbf{o}[t],y),
\end{equation}
where $\alpha$ is also taken as $0.05$. Our experiments show that such loss performs better than that defined in \cref{eqn:appendix_loss} for DVS-CIFAR10.

For all the tasks, we use SGD \cite{rumelhart1986learning} with momentum 0.9 to train the networks, and use cosine annealing \cite{loshchilov2016sgdr} as the learning rate schedule. Other hyperparameters about optimization are listed in \cref{table:parameter}. For ImageNet, we first train
the SNN with only 1 time step to get a pre-trained
model, and then fine-tune the model for multiple time steps.

In Section.~5.3 of the main content, we compare the proposed SLTT method and OTTT \cite{xiao2022online} following the same experimental settings as introduced in \cite{xiao2022online}. For both methods, we adopt the NF-ResNet-34 architecture for ImageNet and VGG-11 (WS) for other datasets. And the
total number of time steps for CIFAR-10, CIFAR-100, ImageNet,
DVS-Gesture, and DVS-CIFAR10 are 6, 6, 6, 20, and 10, respectively.

\section{Societal impact and limitations}
There is no direct negative societal impact since this work focuses on efficient training methods for SNNs. Compared with ANNs, the inference of SNNs requires less energy consumption and produces less carbon dioxide emissions. Our proposed method can further reduce energy consumption in the SNN training process. As for the limitation, the SLTT method cannot be equipped with some network techniques, such as batch normalization along the temporal dimension, since the proposed method is conducted in an online manner, which is biologically more plausible and more friendly for on-chip training. It may require the exploration of more techniques that are compatible with online learning of SNNs.

\end{document}